\documentclass{article}

\usepackage[preprint]{colm2026_conference}

\usepackage[utf8]{inputenc}
\usepackage[T1]{fontenc}
\usepackage{hyperref}
\usepackage{url}
\usepackage{booktabs}
\usepackage{amsfonts}
\usepackage{amsmath}
\usepackage{amssymb}
\usepackage{amsthm}
\usepackage{nicefrac}
\usepackage{microtype}
\usepackage{xcolor}
\usepackage{hyperref}
\usepackage{wrapfig}
\usepackage[pdftex]{graphicx}
\usepackage{float}
\usepackage{subcaption}
\usepackage{fontawesome5}
\usepackage{cleveref} %
\graphicspath{{./}}
\usepackage{pifont}
\setcitestyle{authoryear,round}
\usepackage{etoc}

\usepackage{amsmath,amsfonts,bm}

\def\eqref#1{equation~\ref{#1}}

\def\1{\bm{1}}

\DeclareMathAlphabet{\mathsfit}{\encodingdefault}{\sfdefault}{m}{sl}
\SetMathAlphabet{\mathsfit}{bold}{\encodingdefault}{\sfdefault}{bx}{n}

\usepackage[most]{tcolorbox}
\usepackage{fvextra}

\DefineVerbatimEnvironment{PromptVerbatim}{Verbatim}{
  fontsize=\scriptsize,
  breaklines=true,
  breakanywhere=true
}

\newtcolorbox{promptframe}[1]{
  breakable,
  colback=blue!5,         %
  colframe=blue!25,       %
  colbacktitle=blue!25,
  coltitle=black,
  boxrule=0.5pt,
  arc=1pt,
  left=6pt,
  right=6pt,
  top=6pt,
  bottom=6pt,
  title={#1},
  fonttitle=\bfseries
}

\definecolor{cornflowerblue}{rgb}{0.39, 0.58, 0.93} %
\hypersetup{
    colorlinks=true,
    linkcolor=cornflowerblue,
    filecolor=magenta,      
    urlcolor=teal,
    citecolor=cornflowerblue,%
    }

\def\futuresim{{\texttt{FutureSim}}}

\newtheorem{theorem}{Theorem}[section]

\title{\textsc{FutureSim}: Replaying World Events to Evaluate Adaptive Agents}

\author{%
  \normalfont
	  Shashwat Goel$^{1,2}$\thanks{Equal contribution.} \quad
	  Nikhil Chandak$^{2,4*}$ \quad
	  Arvindh Arun$^{3*}$ \\
	  Ameya Prabhu$^{4,5}$ \quad
	  Steffen Staab$^{3,6}$ \\
	  Moritz Hardt$^{2,4}$ \quad
	  Maksym Andriushchenko$^{1,2}$ \quad
	  Jonas Geiping$^{1,2,4}$ \\
	  \vspace{0.35em}\\
	  $^1$ELLIS Institute T\"ubingen \quad
	  $^2$Max Planck Institute for Intelligent Systems \\
  $^3$Institute for AI, University of Stuttgart \quad
  $^4$T\"ubingen AI Center \\
  $^5$University of T\"ubingen \quad
  $^6$University of Southampton
}

\begin{document}

\maketitle

\vspace{-0.75cm}
\begin{center}
    \begin{tabular}{c@{\hskip 9pt}c}
    \hspace*{0.2cm}\raisebox{-1pt}{\faGlobe} \href{https://openforecaster.github.io/futuresim}{\fontsize{10pt}{0pt}\texttt{Blog}}
    \hspace*{1.4cm}\raisebox{-1pt}{\faGithub} \href{https://github.com/OpenForecaster/futuresim/tree/main}{\fontsize{10pt}{0pt}\texttt{Code}}
\end{tabular}
\end{center}

\vspace{0.25cm}

\begin{abstract}

\looseness -1 AI agents are being increasingly deployed in dynamic, open-ended environments that require adapting to new information as it arrives. To efficiently measure this capability for realistic use-cases, we propose building grounded simulations that replay real-world events in the order they occurred. We build \futuresim, where agents forecast world events beyond their knowledge cutoff while interacting with a chronological replay of the world: real news articles arriving and questions resolving over the simulated period. We evaluate frontier agents in their native harness, testing their ability to predict world events over a three-month period from January to March 2026. \futuresim\ reveals a clear separation in their capabilities, with the best agent's accuracy being $25\%$, and many having worse Brier skill score than making no prediction at all. Through careful ablations, we show how \futuresim\ offers a realistic setting to study emerging research directions like long-horizon test-time adaptation, search, memory, and reasoning about uncertainty. Overall, we hope our benchmark design paves the way to measure AI progress on long-horizon open-ended adaptation spanning multiple months in the real world.
\end{abstract}

\begin{figure}[H]
    \centering
    \begin{subfigure}[t]{0.49\linewidth}
        \centering
        \includegraphics[width=\linewidth]{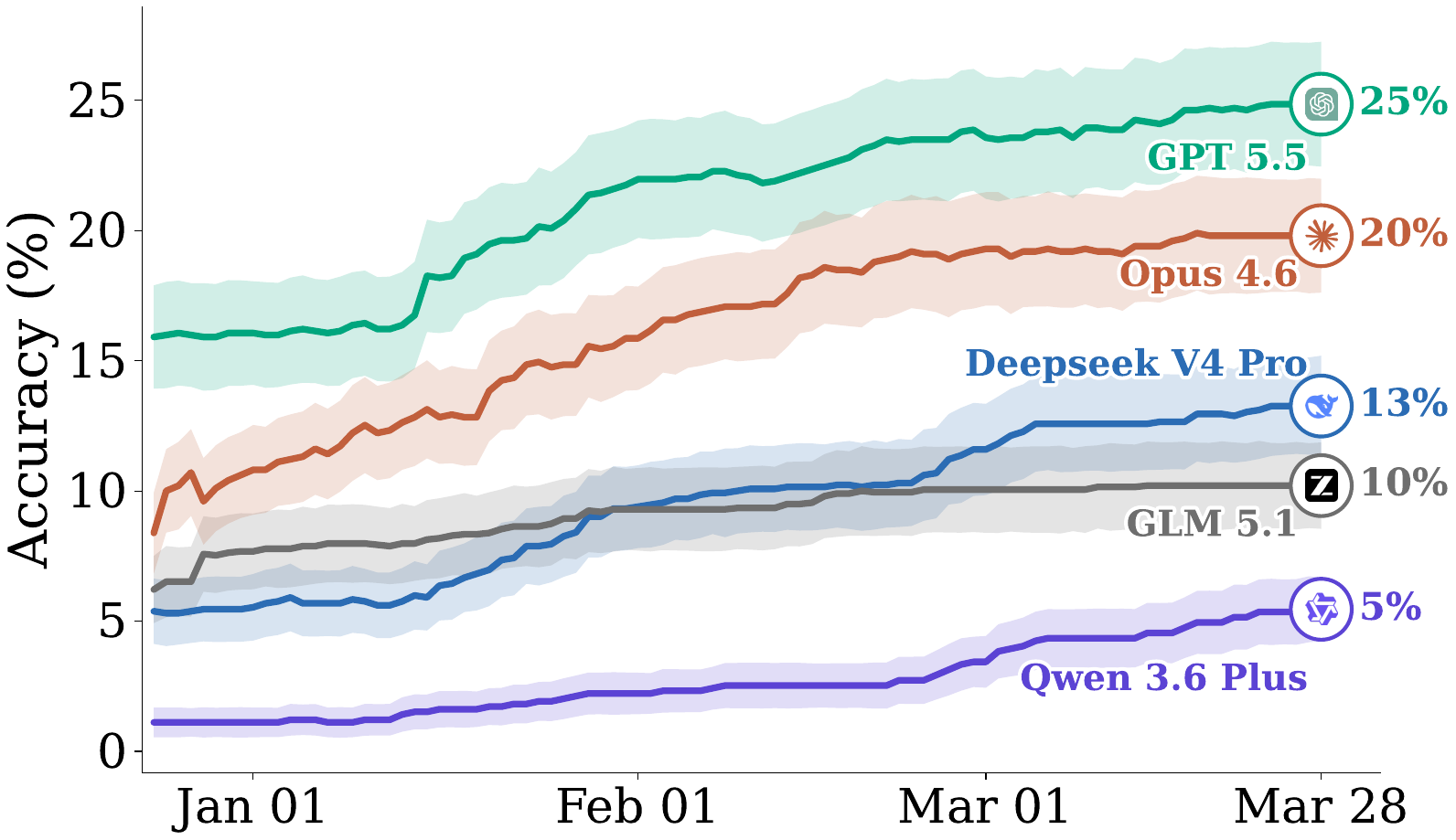}
    \end{subfigure}
    \hfill
    \begin{subfigure}[t]{0.49\linewidth}
        \centering
        \includegraphics[width=\linewidth]{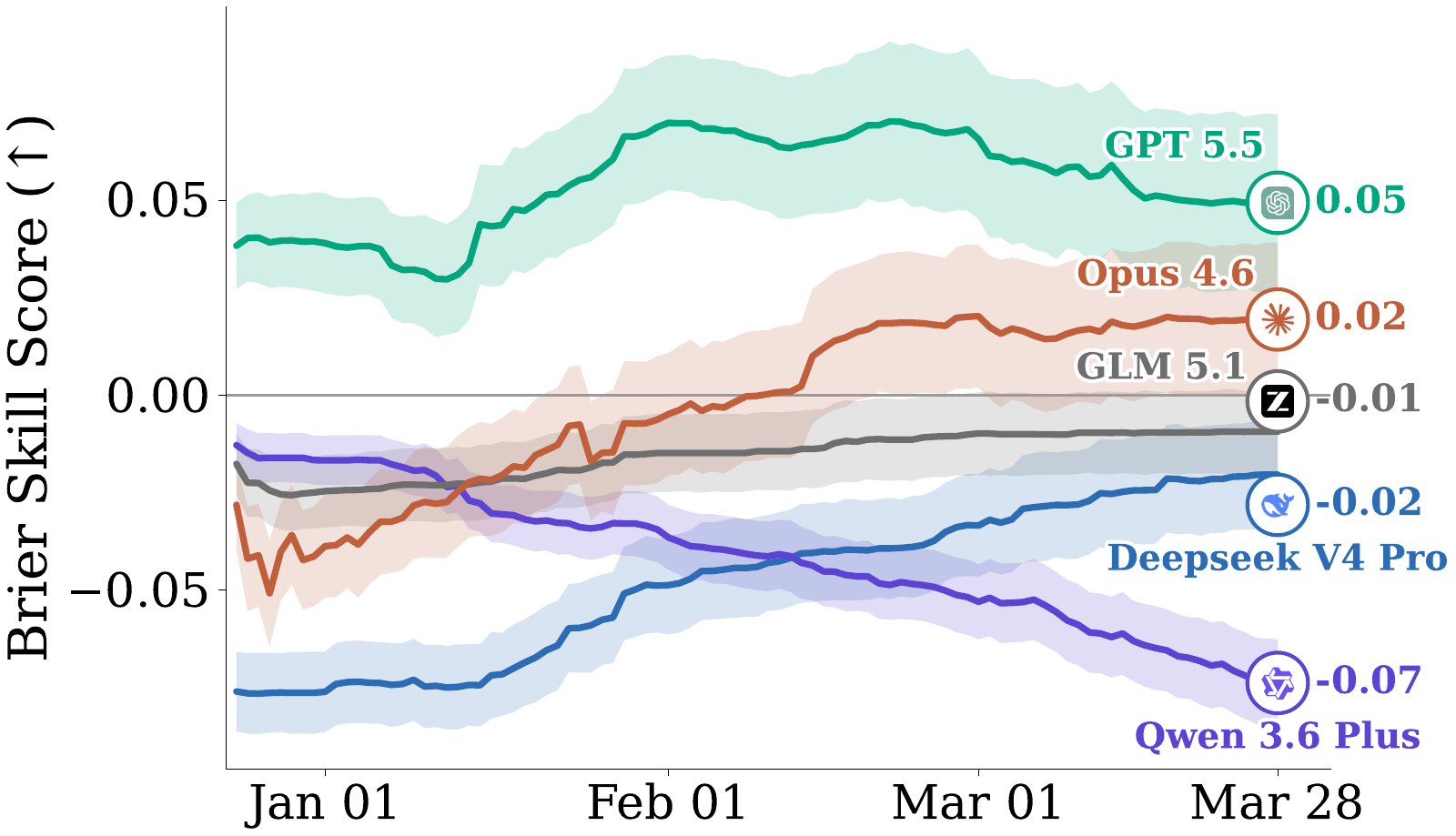}
    \end{subfigure}
    \caption{\looseness -1 In \futuresim, agents have to keep updating their predictions about future world events, by searching over an evolving news corpus up to the current simulation date. We evaluate all models in their recommended harness at maximum reasoning effort over 3 seeds. Models consume over 10M unique tokens and perform 500-4000+ tool calls over the course of the simulation. We see a clear separation in capabilities, with GPT 5.5 performing the best in both top-1 accuracy and Brier skill score. Predictions improve in accuracy over time across all models. However, open-weight frontier models have a negative Brier skill score, which makes them worse than abstaining from prediction altogether.
    }
    \label{fig:figure1}
\end{figure}

\etocdepthtag.toc{mainmatter}

\section{Introduction}
\looseness -1 Language model agents have achieved resounding success across diverse benchmarks. Yet, they are held back by their inability to adapt over a long-horizon in dynamic environments~\citep{foundation2026arcagi3newchallengefrontier, silver2025eraexperience}. While recent game and simulation benchmarks like ARC-AGI 3 and VendingBench~\citep{backlund2025vendingbenchbenchmarklongtermcoherence} provide informative proxies, adaptation should be measured in environments that align with how our world evolves~\citep{farquhar2019robustevaluationscontinuallearning}.
This raises the question:

\begin{quote}
    \textit{Can we evaluate adaptation in settings that require reasoning about how our world evolves?%
    }
\end{quote}

To address this, we introduce \futuresim, an interactive environment where agents have to predict world events beyond the underlying language model's knowledge cutoff. 
\futuresim\ replays real-world events in the temporal order they occurred to evaluate adaptive agents: each day, agents can search new information from news articles, and obtain ground-truth feedback on some questions, based on which they can keep updating their predictions.

Forecasting is a useful testbed for studying long-horizon test-time adaptation in two key ways. 
First, it assesses whether agents' priors align with how the world evolves, and whether they appropriately update them in light of new information. Indeed, a decade of research on human expert forecasters has found that they make continuous updates to their beliefs which leads to more accurate forecasts~\citep{mellers2015identifying, atanasov2020small}. Second, a predictive setting ensures agents' actions cannot change the underlying environment, allowing us to stay grounded in real-world events rather than approximately simulating counterfactual worlds~\citep{zhou2026mindsim2realgapuser}.

\looseness -1 For our experiments, we show such simulations can be bootstrapped from just timestamped source documents. We create prediction questions from news articles by refining the methodology proposed in~\cite{chandak2026curating}, and evolve the context as daily news arrives. We benchmark agents in a 90-day period from January to March 2026 over 330 events in \Cref{fig:figure1}. 
We find that \futuresim\ cleanly discriminates the capabilities of both closed and open-weight frontier models in their recommended agent harnesses. 
All models utilize new information to improve their predictions over the course of the simulation. The best performing agent, GPT 5.5 in Codex, consumes 3700 turns and 12.4M tokens spanning multiple sequential context window compactions in a single run showing the long-horizon nature of \futuresim.

Overall, our main contributions are as follows:

\begin{itemize}
    \item \looseness -1 To realistically evaluate adaptation in AI agents on the economically valuable task of forecasting, we build \futuresim: an environment replaying how the world evolved after the agents' knowledge cutoff. To prevent leakage of future information, we carefully sandbox agents, while providing them access to reliably dated offline snapshots of news articles.
    
    \item \looseness -1 \futuresim\ is open-ended: agents choose which questions to make forecasts on, and when. The questions are free-form, and agents have to come up with multiple possible outcomes with an incentive to report calibrated probability distributions and make timely updates. The environment itself is flexible, enforcing only two actions--- \texttt{submit\_prediction()} and \texttt{next\_day()}. Users can then specify their own combination of models, harnesses, and chronological event data.
    
    \item \looseness -1 We show significant room for improvement in test time adaptation for frontier models, even in their native agent harnesses. We design experiments that show how \futuresim\ can also support research on emerging capabilities like reasoning to search under uncertainty, memory, harness design, and multi-agent dynamics.

\end{itemize}

\section{Related Work}

\looseness -1 Prior work on realistic continual learning evaluations for language models emphasises using real-world time as the underlying distribution shift~\citep{lazaridou2021mind, yao2022wildtime}. We differ in evaluating agents on forecasting future events, targeting the capability to reason under uncertainty and take information-seeking actions over a long horizon.

\begin{table}[t]
\centering
\small
\setlength{\tabcolsep}{3.5pt}
\renewcommand{\arraystretch}{1.12}
\newcommand{\cmark}{\textcolor{green!60!black}{\ding{51}}}
\newcommand{\xmark}{\textcolor{red!70!black}{\ding{55}}}
\caption{\textbf{Comparison to related benchmarks.} \futuresim\ is the only reproducible benchmark that tests agents on open-ended adaptive reasoning about real-world events over a long horizon. We explain the methodology for each column in~\Cref{app:tableexplanation}.}

\resizebox{\linewidth}{!}{
\begin{tabular}{lccccc}
\toprule
\textbf{Benchmark} &
\textbf{World reasoning} &
\textbf{Reproducible} &
\textbf{Tests adaptation} &
\textbf{Open-ended} &
\textbf{Horizon Length} \\
\midrule
GAIA-2~\citep{froger2026gaia}              & \xmark & \cmark & \xmark & \xmark & 25 \\
$\tau^3$-Bench~\citep{shi2026tauknowledgeevaluatingconversationalagents}      & \xmark & \cmark & \xmark & \xmark & 33 \\
ARC-AGI-3~\citep{foundation2026arcagi3newchallengefrontier}           & \xmark & \cmark & \cmark & \cmark & 7,800 \\
BALROG~\citep{paglieri2025balrog}              & \xmark & \cmark & \cmark & \cmark & 100,000 \\
SWE-Evo~\citep{thai2026sweevobenchmarkingcodingagents}       & \xmark & \cmark & \cmark & \xmark & N/R \\
Vending-Bench-2~\citep{backlund2025vendingbenchbenchmarklongtermcoherence}     & \xmark & \xmark & \cmark & \cmark & 6,000 \\
KellyBench~\citep{grady2026kellybenchbenchmarklonghorizonsequential}          & \xmark & \cmark & \cmark & \cmark & 1,000 \\
ForecastBench~\citep{karger2025forecastbench}       & \cmark & \cmark & \xmark & \xmark & 1 \\
ProphetArena~\citep{yang2026llmasaprophet}        & \cmark & \xmark & \xmark & \xmark & 1 \\
PredictionArena~\citep{zhang2026predictionarenabenchmarkingai}     & \cmark & \xmark & \cmark & \cmark & N/R \\
\midrule
\textbf{\futuresim\ (Ours)} & \cmark & \cmark & \cmark & \cmark & 4,000 \\
\bottomrule
\end{tabular}
} 
\label{tab:benchmark-comparison}
\end{table}

\looseness=-1 \textbf{Long-horizon environments.} 
Recent benchmarks study long-horizon agent interaction in games and simulations. For example, BALROG~\citep{paglieri2025balrog} uses games to evaluate open-ended multimodal agents, while ARC-AGI 3~\citep{foundation2026arcagi3newchallengefrontier} uses interactive games for testing learning on the fly. GAIA-2~\citep{froger2026gaia} evaluates agents in human-authored mobile-use scenarios, while Vending-Bench~\citep{backlund2025vendingbenchbenchmarklongtermcoherence} and YC-Bench~\citep{he2026textttycbenchbenchmarkingaiagents} study agents in economic simulations. The dynamics in these environments are either human-designed or simulated by a model, making them imperfect proxies for societal evolution~\citep{seshadri2026lostsimulationllmsimulatedusers}. Live deployments such as Anthropic's Project Vend~\citep{anthropic2025projectvend} provide the most realism, but are slow, expensive, can raise safety concerns, and are not exactly reproducible across models. To circumvent such issues, we propose drawing environment dynamics from real-world event timestamps. This is related to benchmarks that replay the evolution of real artifacts, such as SWE-Evo~\citep{thai2026sweevobenchmarkingcodingagents} for coding agents, but we show this can be generalized across domains by creating prediction tasks from timestamped source documents~\citep{chandak2026curating}.

\looseness=-1 \textbf{Judgemental Forecasting benchmarks.}
Language model forecasting has recently received substantial attention in the domain of \textit{judgemental forecasting}, i.e., predicting discrete real-world events, unlike \textit{statistical forecasting} of time-series data~\citep{WEBBY199691}. Early evaluations~\citep{halawi2024approachinghumanlevelforecastinglanguage, karger2025forecastbench} tested models on static forecasting questions with retrieved evidence. More recently, benchmarks such as ProphetArena~\citep{yang2026llmasaprophet} and PredictionArena~\citep{zhang2026predictionarenabenchmarkingai} evaluate agents through live trading on prediction markets like Kalshi and Polymarket. \futuresim\ can subsume such evaluations by allowing benchmarking on any data, including theirs, while having a crucial advantage: live-market evaluations are difficult to reproduce and ablate because market conditions change over time, and reliable signals require long evaluation windows. \futuresim\ instead provides a replayable environment that allows controlled ablations of search, memory, harness design, and adaptation. Parallel work KellyBench~\citep{grady2026kellybenchbenchmarklonghorizonsequential} studies long-horizon betting on Premier League matches through statistical model building. In contrast, \futuresim\ evaluates agents’ ability to forecast diverse, general world events from evolving evidence.

\section{The \futuresim ~Environment}
\label{sec:futuresim}

\futuresim\ is a chronological environment, and in our work, we set a daily cadence, with time-steps corresponding to real-world dates. We now describe the mechanics of \futuresim.

\begin{figure}[t]
    \centering
    \includegraphics[width=0.7\linewidth]{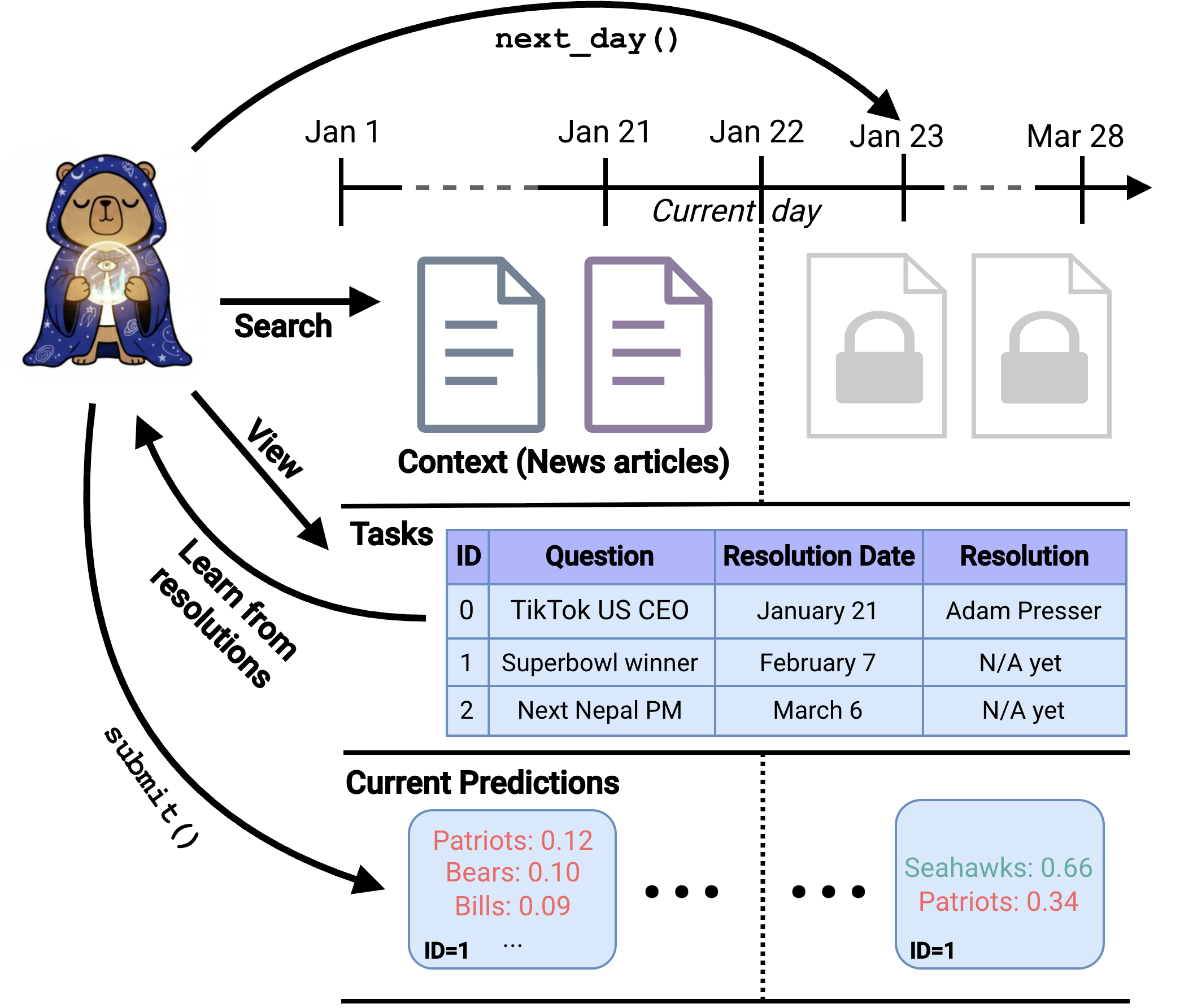}
    \caption{In \futuresim, agents are evaluated in a dynamic forecasting environment. They can search news up to the current simulation date (access to future information is restricted), gather feedback from resolved questions and choose when to update their predictions. The environment only enforces two actions: \texttt{submit()} required to submit or revise predictions for a question and \texttt{next\_day()} to advance the simulation by one day. 
    }
    \label{fig:simulation}
\end{figure}

\subsection{Environment Design}

The task state and context state of the environment are updated at each time step as follows.

\looseness -1 \paragraph{Tasks.}
\looseness=-1 The current state of tasks is maintained as a CSV file, with each row containing data about one forecasting question (see \Cref{app:sample-question} for an example). This includes the question's background information, resolution criteria, resolution date, and the agent's most recent forecast for the question. Unlike prior forecasting evaluations that expect a single predicted outcome per question, we allow agents to submit a probability distribution over multiple possible outcomes that it has to come up with itself. Given the inherent uncertainty in forecasting, this allows a more complete measurement of the agent's beliefs. Once the simulation date passes a question's resolution date, the ground-truth outcome is added to the state file. 

\paragraph{Context.} At each time-step, the context consists of documents that became available by then. We use Common Crawl News~\citep{nagel2016-ccnews} as it provides reliably dated snapshots of news articles which are not affected by future updates~\citep{chandak2026curating}.

\textbf{Agent Interaction.} The environment itself is minimal, allowing the definition of arbitrary \textit{agents}. By an agent, we mean the tools, prompts, and orchestration that models can use to interact with the environment, which we describe later. The environment itself provides only two actions:
\begin{enumerate}
    \item \texttt{submit\_forecast(question\_id, outcomes)}: register or update the probability distribution over predicted outcomes for one active question.
    \item \texttt{next\_day()}: ends the current time-step and advances the simulation, running any necessary evaluations. The task and context state are updated for the next time-step.
\end{enumerate}

\paragraph{Evaluation.} Let $\mathcal{Q}$ denote the set of forecasting questions, with each $q \in \mathcal{Q}$ having a ground-truth answer $y_q$. Formally, an agent's forecast is a set of outcomes $\Omega_q = \{o_1, \ldots, o_k\}$ with probabilities $p_q(o_i) \ge 0$ satisfying $\sum_{i=1}^{k} p_q(o_i) \le 1$. We use language model-based answer matching~\citep{chandak2025answermatchingoutperformsmultiple} to check if an outcome $o$ matches the ground-truth $y_q$. We report the following metrics:

\textbf{Brier Skill Score (BSS).}
For our free-form questions with no predefined set of choices, past work \citep{damani2026beyond, chandak2026curating} has defined Brier score for the case when only a single outcome is provided by the forecaster. However, reasoning about the future inherently involves modeling multiple possible outcomes. In this work, we define Brier Skill Score to incorporate that by adapting the existing multi-category Brier score \citep{mucsanyi2023proper} as follows:
\[
\mathrm{BSS}(q)
=
1 -
\sum_{o \in \Omega_q \cup \{y_q\}}
\left(p_q(o) - \mathbf{1}[o = y_q]\right)^2
\]
where $p_q(o) = 0$ for $o \notin \Omega_q$. In \Cref{app:multibrier}, we prove that this is a proper scoring rule. Higher score is better: $1$ is for a fully confident correct answer, $0$ for abstaining by reporting zero probability, and $-1$ is for assigning all probability to wrong guesses. A worked example is provided in \Cref{app:bss-example}.

\textbf{Accuracy.}
We measure the accuracy of the outcome assigned the most probability (top-1 accuracy)
\[
\mathrm{Acc}
=
\frac{1}{|\mathcal{Q}|}
\sum_{q \in \mathcal{Q}}
\mathbf{1}
\left[
\arg\max_{o} p_q(o) = y_q
\right]
\]

\looseness=-1 While the Brier skill score incorporates both the correctness of predicted outcomes and the calibration of the probabilities distributed over them, the accuracy is agnostic of the probability distribution and only depends on the correctness of the top outcome predicted for each question. In our results, we report the mean across all questions $\mathcal{Q}$, reporting the projected score based on the current prediction for unresolved questions, and scoring $0$ for questions where an agent has registered no prediction.

\section{Benchmarking Frontier Agents in \futuresim}
\label{sec:benchmark}

We now describe the current \futuresim\ benchmark setup and results on frontier agents.

\subsection{Experimental setup}

\textbf{Forecasting Questions.} Forecasting evaluations require fresh questions about events that occur after the knowledge cutoffs of models being tested. We follow the scalable automated methodology proposed by \citet{chandak2026curating} to curate short-answer forecasting questions without predefined choices, from news articles. We make multiple refinements to this pipeline as described in ~\Cref{app:questiondetails}, such as using stronger models, filtering out easy questions, and improving the reliability of resolution dates. We use Al~Jazeera articles as source documents, as we found this to be the highest quality source with a large number of articles freely accessible via CCNews. Finally, we obtain 330 forecasting questions resolving between January 1st and March 28th 2026, which is after the knowledge cutoffs of the models we evaluate. Each question is active from the start of the simulation (2025-12-24) until its resolution date. Agent predictions to each question are evaluated by DeepSeek v3.2 as the answer matcher, with prompts in \Cref{app:answer-matching-prompts}. We limit the outcome set size predicted by agents to $\leq 5$ per question for answer matching efficiency.

\textbf{Search Corpus.} The search corpus is a deduplicated snapshot of CCNews containing 7.36M articles from 141 distinct news sources between January 2023 and March 2026. Agents have access to articles published on or before the current simulation date. About 7.12M articles are available on day 0, and 244K new articles enter the corpus over the 88-day simulation window. We rely on a local offline news corpus instead of external search APIs like Brave as we found issues with their date filtering, with an example provided in \Cref{app:contextdetails}.

\looseness=-1 We provide the topic distribution of questions in \Cref{app:topic-distribution}, noting it is more representative of important events than prediction markets. That said, the prediction questions and search corpus can be changed easily in \futuresim, and we will keep updating both for new models.

\textbf{Tools available to agents.} All agents are run with full access to the harness's built-in shell and file tools (eg: \texttt{Bash}, \texttt{Read}, \texttt{Write}, \texttt{Grep}). These commands can be used to access the raw articles, which are organized into folders by date for ease of browsing, and create, edit or execute files in the agents own workspace. To prevent leakage of future information, we do not allow even read access beyond the agent's workspace and accessible article folders up to the current simulation date and disable web search tools and commands like \texttt{curl}, with sandbox details in \Cref{app:sandboxdetails}. Finally, we provide access to a hybrid semantic + keyword search tool over the news corpus, implemented using LanceDB, that returns 5 chunks of 512 tokens. We use Qwen3 8B embeddings for the semantic search. The search tool allows defining the period between which to fetch chunks, in the format \texttt{search\_news(query, from\_date, to\_date)}.

\textbf{Agents tested.} Unless specified, we evaluate at maximum reasoning effort, in the recommended harness for open-weight models, and native code harness for closed ones: We use Codex for GPT 5.5; OpenCode for Qwen3.6 Plus; and Claude Code for Opus 4.6, DeepSeek V4 Pro, and GLM 5.1. We use Claude Code version 2.1.132, Codex version 0.125.0, and OpenCode version 1.4.11. For DeepSeek and Qwen models, we add additional prompting to encourage them to update predictions, as in our initial runs they would only submit on first day and just keep calling \texttt{next\_day()}.

\begin{figure}[t]
    \centering
    \includegraphics[width=0.95\linewidth]{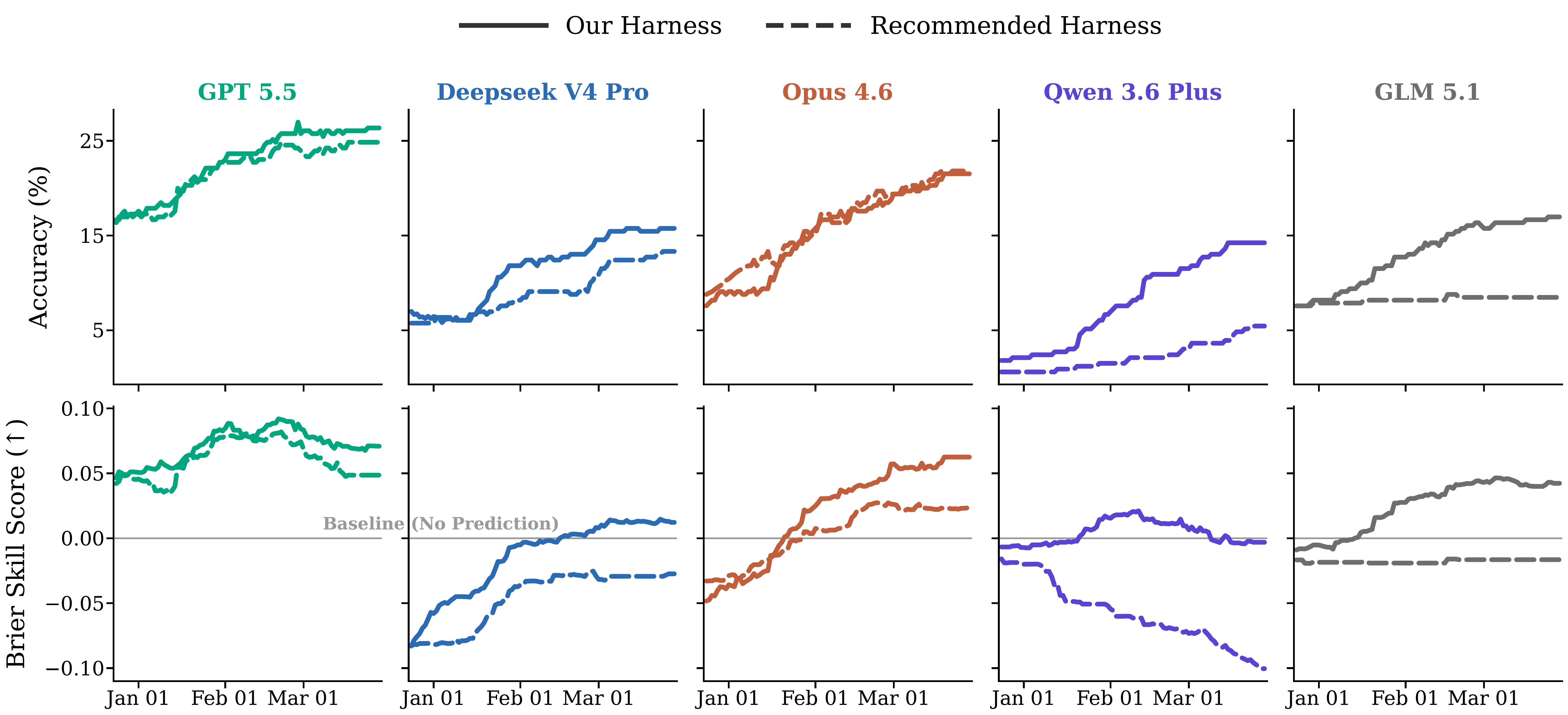}
    \caption{\textbf{Agent performance on \futuresim.} GPT 5.5 by far performs the best in both accuracy and Brier skill score. (Top) Open-weight models achieve significantly higher accuracy in our modified harness, improving consistently over the course of the simulation. (Bottom) Except for GPT 5.5, all models start at a negative Brier skill score, and while they fail to improve significantly by default, they cross over to a positive Brier skill score in our harness.}
    \label{fig:harness}
\end{figure}

\textbf{Our Baseline Harness.} While it is interesting to study agent performance without any task-specific hand-holding~\citep{foundation2026arcagi3newchallengefrontier}, better orchestration of tools and prompts can help elicit agent potential on any given task~\citep{wang2025openhands}. Thus, we also build a custom harness as an improved baseline for future work on our benchmark. For developing our harness, we iterated by testing features on DeepSeek V3.2 performance on OpenForesight (May to August 2025)~\citep{chandak2026curating} as a validation setting. We then overlay the features developed on top of native harnesses which includes context consumption feedback, procedural forecasting guidelines, loading task state as a dataframe so agents can process it with python query execution, a forced memory update phase with structured memory tools and per-question memory. We provide details in Appendix~\ref{app:harness}.

\begin{figure}[t]
    \centering
    \includegraphics[width=\linewidth]{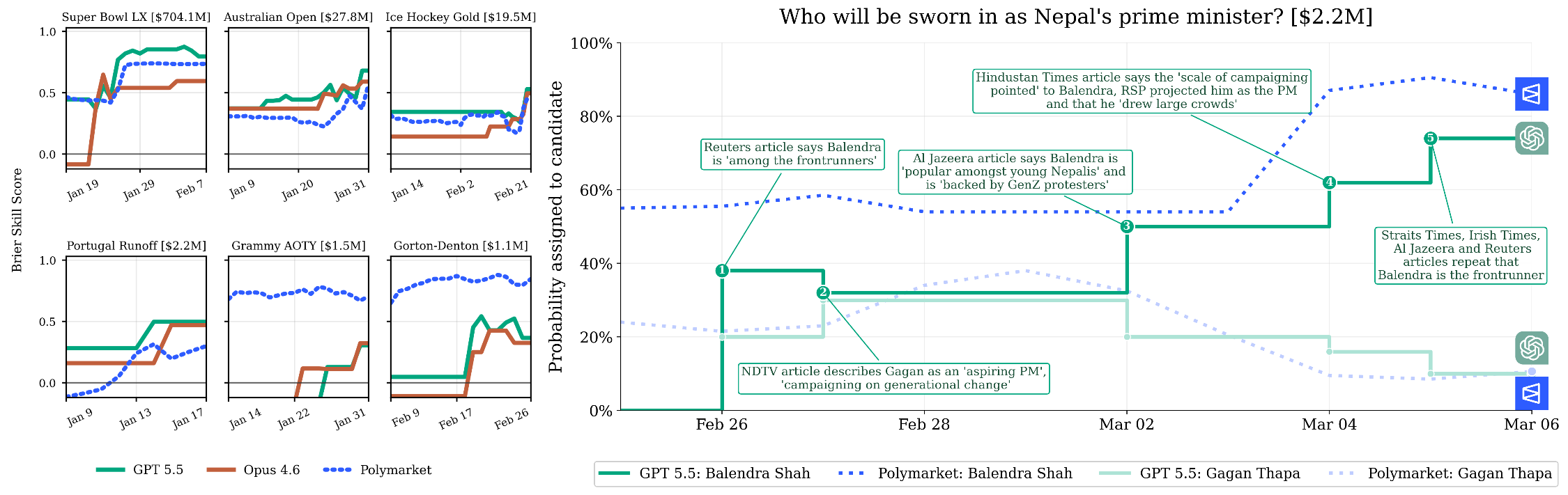}
    \caption{\looseness -1 \textbf{Comparisons of frontier agent prediction to human crowd aggregates on Polymarket.} (Left) We find that GPT 5.5 leads the real market aggregate for some questions, including the Super Bowl market, which traded $700$M in total volume. That said, it performs relatively poorly on some other markets, with Claude Opus 4.6 closely tracking GPT 5.5 predictions but usually slightly worse. (Right) We zoom in on the market about the Nepal PM election, annotating rationales from GPT 5.5's prediction update trajectory. Notably, GPT 5.5 cites reasonable evidence, and its updates on candidate probabilities are aligned with, albeit lagging the human aggregate.}
    \label{fig:polymarket}
\end{figure}

\subsection{Results}

\Cref{fig:harness} shows the Brier skill score and accuracy over the course of the simulation for different models, both in their default-recommended harness and with our harness. GPT 5.5 leads by a large margin on both metrics, starting with the best day 0 predictions and improving throughout. DeepSeek V4 Pro also improves substantially in both harnesses, whereas Qwen3.6 Plus and GLM 5.1 improve only in our harness; Qwen3.6 Plus even deteriorates in calibration in its recommended OpenCode harness. These gaps show that good harness design can elicit test-time adaptation even from open-weight models. Claude Opus 4.6 also shows better calibration in our harness than native Claude Code, suggesting coding harnesses may not be generally optimal.

\looseness=-1 We believe further harness engineering can likely elicit stronger scores on \futuresim, like any benchmark. However, \futuresim involves long-horizon experiments before final metrics can be measured, making it a challenging test-bed for emerging context management methodologies like recursive language models~\citep{zhang2026recursivelanguagemodels}, and autoresearch approaches that let agents self-improve their own harness~\citep{zhang2026hyperagents, lou2026autoharnessimprovingllmagents}.

\vspace{-0.1cm}
\subsection{Analyzing Agent Behaviours}

\looseness -1 We now analyze the trajectories and reasoning traces of agents and report some interesting strategies and failure modes we observe. GPT 5.5 has the best Top-1 accuracy, but we find among its incorrect final predictions, 27.4\% assign at least 0.5 probability to the wrong top answer, and 9.1\% assign at least 0.75, indicating significant overconfidence. DeepSeek V4 pro updates its prediction on almost every question during the simulation, which helps it improve substantially from its initial weak predictions. However, it often places null predictions like \textit{``no new appointment''} which reduce its score.

GLM 5.1 is relatively conservative, with the smallest overconfidence rate among wrong answers (only 3.7\% above 0.5), but performs poorly as it makes the lowest number of updates among the models tested. Qwen3.6 Plus abstains from more questions than any other model, only registering predictions on 36.7\% questions, while all other models submit predictions on all the questions. Note that even when we subset on the questions which all agents predict on at least once (\Cref{app:subset_questions}), the agent rankings are still consistent with \Cref{fig:figure1}. Across models, we observe instances of ``self-conditioning'' \citep{sinha2026the}, where models start treating prior rationales and lessons they wrote in memory as hard-truths, leading to subsequent overconfident mistakes.

\paragraph{Comparing Agents to Human Aggregate Forecasts.} While we use a generalizable recipe to create forecasting questions from news reporting, some questions also have corresponding prediction markets on Polymarket. In Figure~\ref{fig:polymarket} (left), we compare model predictions to the human aggregate on these questions at corresponding times. We find that GPT 5.5's prediction updates are not only aligned with movements in the human aggregate for multiple markets, sometimes they even lead the market (e.g. predicting the Super Bowl and Portugal Runoff winner), which could be potentially used to make large profits given the millions in volume traded on these markets. At the same time, the model's predictions are also significantly worse on the Grammy and UK constituency election (Gorton and Denton) markets which anecdotally both aggregate human preferences. 

\looseness=-1 In Figure~\ref{fig:polymarket} (right), we visualize GPT 5.5's update trajectory on the Nepal Prime Minister market, finding its updates appropriate based on the cited evidence. Naturally, given the millions in volume traded for these questions on Polymarket, the human aggregate updates are smoother than a single agent run, but GPT 5.5's updates are closely aligned with their movements, if only a bit lagging as our search corpus is not as fresh, notably missing social media posts. Overall, this analysis shows that our simulation is realistic, and tracks economically valuable capabilities.

\section{Capabilities Tested}
\label{sec:capabilities}
\vspace{-0.1cm}
\looseness -1 Performing optimally on \futuresim\ requires agents to be capable on various fronts: agents should focus on the right questions based on the new context available at each time-step, creatively search for relevant evidence, remember useful information over the course of the simulation, and learn from environment feedback as past predictions resolve. In this section, we show how \futuresim\ ablations can isolate the effect of these capabilities to enable research on these emerging directions.

\begin{figure}[t]
    \centering
        \includegraphics[width=0.7\linewidth]{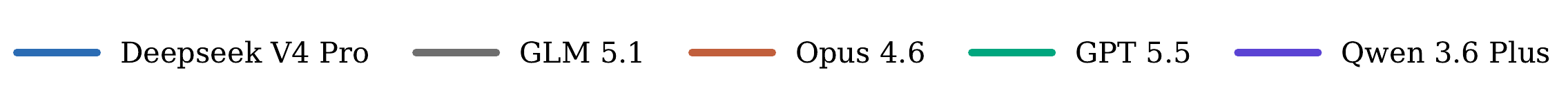}
    \makebox[\linewidth][c]{%
        \begin{minipage}[t]{0.45\linewidth}
            \centering
            \includegraphics[width=0.8\linewidth]{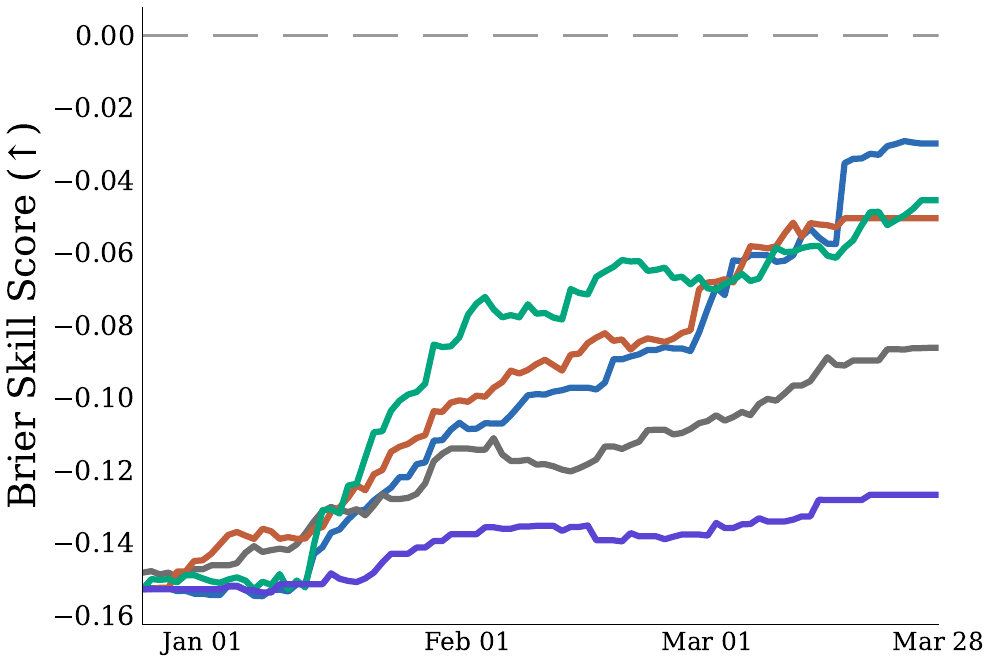}
        \end{minipage}\hspace{0.03\linewidth}%
        \begin{minipage}[t]{0.45\linewidth}
            \centering
            \includegraphics[width=0.9\linewidth]{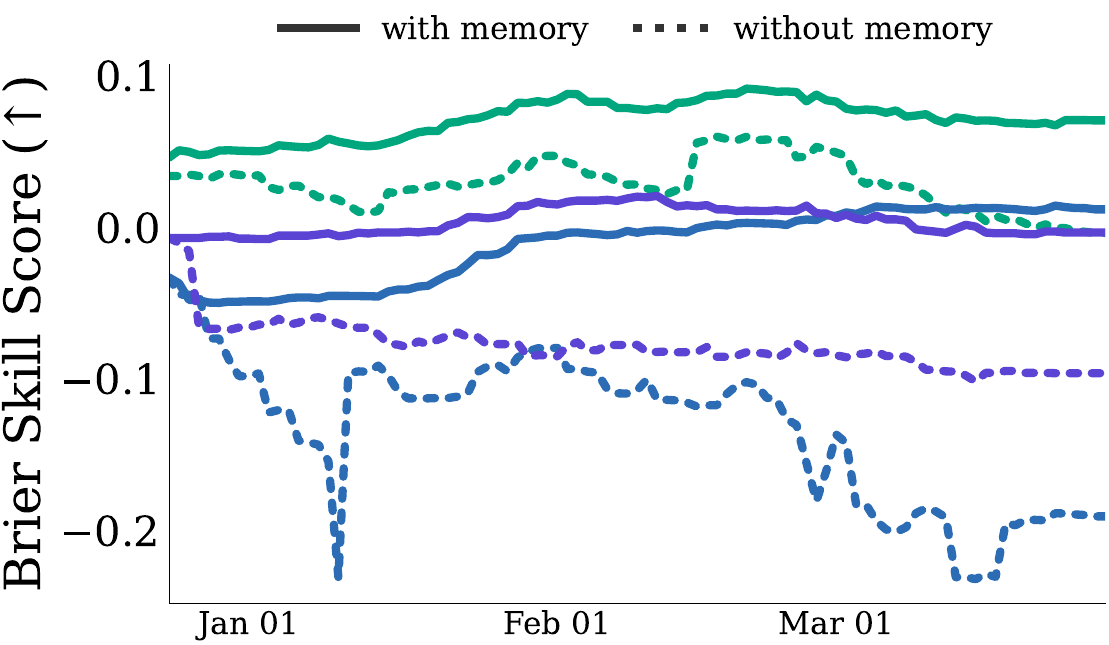}
        \end{minipage}%
    }
    \caption{\textbf{(Left) Comparing test-time adaptation across agents.} We start different models in our improved harness at the lowest performing agent's (Qwen 3.6 Plus) initial prediction set, to maximize scope for improvement over time. We find agents get anchored to the initial predictions, failing to adapt them sufficiently to even reach the no prediction baseline of 0 brier skill score, even when their own capabilities are stronger. \textbf{(Right) Benefits from memory.} By ablating the ability to write and fetch memories at test time, we find models clearly benefit from inference-time memory.}
    \label{fig:adaptation-memory}
\end{figure}

\subsection{Test-time adaptation}

Each day, as the available context grows in the simulation, agents should update their forecasts based on the new information. Moreover, as some forecasts are resolved on the date the ground-truth becomes known, agents can reflect on their predictions to learn forecasting meta-skills, akin to on-the-job learning. This leads to two natural research questions: \textit{1) How do different agents compare in their ability to perform test-time adaptation?} and \textit{2) How can we know, for a single agent, how much room there is to improve at test time adaptation?}

\looseness=-1 \textbf{Comparing test-time adaptation across agents.} In the standard setup of \futuresim\ presented until now, each agent has a different starting performance, making it hard to compare test-time adaptation ability across agents. To isolate and compare adaptation, in the left panel of \Cref{fig:adaptation-memory}, we fix the initial forecasts to the worst agent's (Qwen3.6 Plus) and evaluate all agents in our harness, which maximizes the scope for improvement at test-time. We observe that GPT 5.5, Claude Opus 4.6, and DeepSeek V4 Pro all show similar levels of improvement, while GLM 5.1 does worse relative to them, and finally Qwen3.6 Plus barely improves its Brier skill score over time. At the end of the simulation, all agents fail to reach even the baseline Brier skill score (0) of not predicting at all, let alone their original positive Brier skill score performance. This is despite being informed that their predictions are obtaining a negative Brier skill score as they resolve, showing how frontier agents fail to adapt away from bad initial anchors.

\looseness -1 \textbf{Can agents match full knowledge performance with sequential arrival of information?} During our simulation, each day new news arrives sequentially, and agents have to incorporate it by updating forecasts up to one day before the resolution of each question. One way to test how efficiently agents update their predictions is by checking whether they can match their own performance when directly asked to predict each question individually, one day before its resolution date. In this setting, agents have maximal context available for their forecasts, no earlier predictions to anchor to, and are separately prompted to focus on each question individually. We perform this experiment for GPT 5.5 xhigh in Codex, where comparing the left bars of both subplots in \Cref{fig:search}, we find that directly searching with full information available (green) leads to much higher accuracy (31.2\% vs 24.8\%) than sequentially updating predictions as information becomes available in the simulation (blue). This showcases a clear inefficiency in the agent's ability to adapt its predictions at test time and also shows a lowerbound on the maximum accuracy achievable in the simulation. 

\looseness=-1 We hope these results spark work on test-time adaptation in models using \futuresim, including continued finetuning~\citep{yang2025synthetic} to internalize the large amount of new news arriving in the simulation, and test-time training on forecast resolutions~\citep{hardt2024testtime}. 

\vspace{-0.2cm}
\subsection{Memory}

Given the limits of the context window size in transformer-based language models, agents need alternate ways to implement long-term memory. The current approach to frontier agent deployments uses a file-based memory with structured access, which we implement in our harness with carefully designed, minimal yet informative guidelines. We ablate agents having write access to any memories in the right panel of \Cref{fig:adaptation-memory} and find that all three models tested perform worse without memory. This shows the importance of memory for our task, as observed from qualitative traces: agents store and use post-resolution feedback, information found via search over the context, and summaries of past rationale. Memory also protects against drift: when searches returned weak or stale evidence, agents with memory often retained a calibrated prior rather than re-reasoning from scratch and over-updating toward a plausible but unsupported alternative.

\begin{wrapfigure}[17]{r}{0.45\linewidth}
\vspace{-1.1em}
    \centering
    \includegraphics[width=\linewidth]{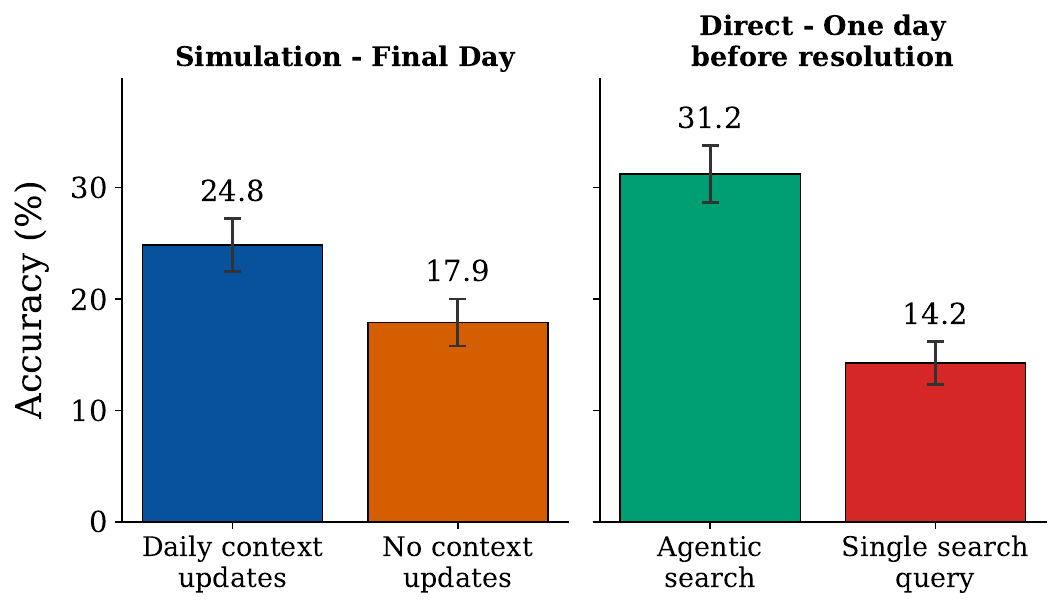}
    \caption{\looseness=-1 \textbf{Benefits from search.} We evaluate GPT 5.5 at xhigh reasoning effort in four different settings, showing the large benefits of agentic search (green vs red) and utilizing the evolving context corpus in \futuresim. 
    }
    \label{fig:search}
\vspace{-1.0em}
\end{wrapfigure}

\subsection{Search}

\looseness=-1 Unlike many existing search benchmarks~\citep{yang-etal-2018-hotpotqa, wei2025browsecomp}, in \futuresim, the search is not for past facts knowable perfectly from the accessible documents. Rather, (i) the document corpus evolves, adding more context each day, and (ii) agents have to creatively reason to come up with \textit{what to search for}, based on the scattered evidence across documents seen until then. To measure the effect of these properties on \futuresim\ performance, we ablate full agentic search over the evolving corpus in the simulation in two key ways, with results in \Cref{fig:search}. First, we ablate the daily arrival of news articles during the simulation, finding that this leads to significantly worse accuracy on the last day (24.8\% in blue with daily updates vs 17.9\% in orange without context updates). This shows the importance of continued search for fresh evidence as the corpus evolves in the environment. Second, we compare full agentic search (green) to retrieving only articles with a single semantic search query using the question title (red), where for each question, we perform the search 1 day before it resolves. We find full agentic search leads to double the accuracy of a single query, demonstrating the importance of sequential information-seeking for forecasting (Brier skill score in \Cref{app:searchablationbrier}).

\begin{wrapfigure}[14]{r}{0.45\linewidth}
\vspace{-1.1em}
    \centering
    \includegraphics[width=\linewidth]{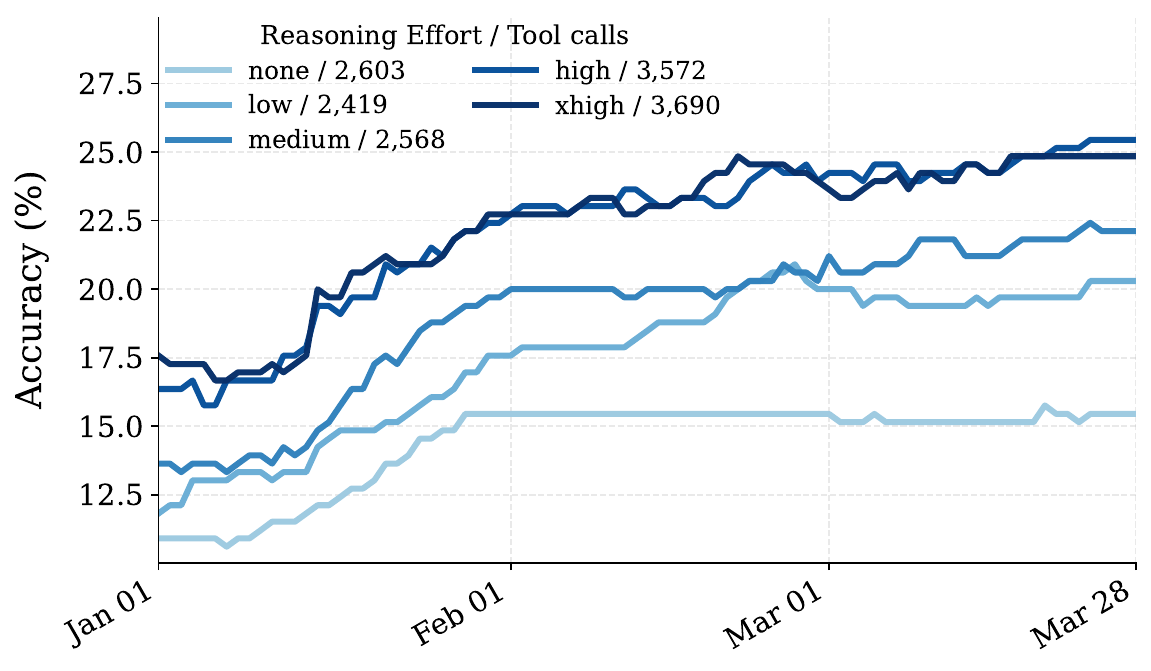}
    \caption{\textbf{Benefits from scaling test-time compute.} We run GPT 5.5 in all available reasoning efforts, finding more inference compute leads to better accuracy on \futuresim .}
    \label{fig:inference-scaling}
\vspace{-1.0em}
\end{wrapfigure}

\looseness -1 These results show \futuresim\ can support research on reasoning-intensive sequential search agents~\citep{jin2025searchr1}, as well as better underlying search tools~\citep{khattab2020colbertefficienteffectivepassage, shao2025reasonir} for the dynamic, and uniquely bayesian search setting of forecasting~\citep{murphy2026agenticforecastingusingsequential}.

\subsection{Inference Scaling}

For economically valuable tasks like forecasting world events, scaling test-time compute can provide beneficial gains in performance worth the extra cost~\citep{snell2025scaling}. To demonstrate this, we compare GPT 5.5 at five different reasoning efforts in \Cref{fig:inference-scaling}, observing consistent improvements in accuracy at higher reasoning budgets (similar trends hold for brier score in \Cref{app:inferencescalingbrier}). We note that GPT 5.5 at high effort uses much more tool calls and obtains significantly higher accuracy than lower efforts, though xhigh does not lead to further improvements over high.

\begin{wrapfigure}[24]{r}{0.45\linewidth}
\vspace{-1.1em}
    \centering
    \includegraphics[width=\linewidth]{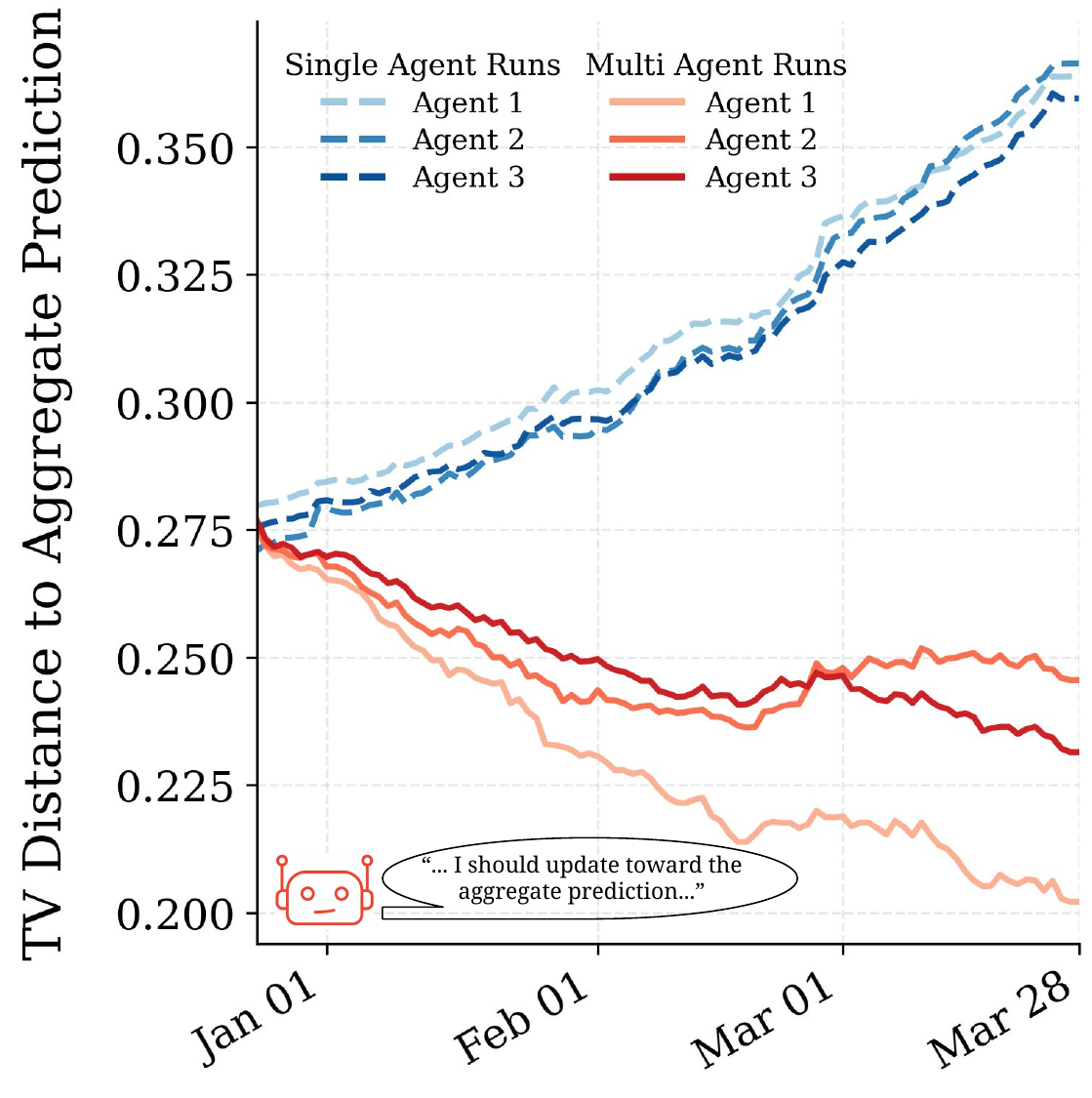}
    \caption{\textbf{Multi-agent dynamics.} When we run multiple copies of DeepSeek V3.2 agents simultaneously, we see agent predictions start moving toward the aggregate, unlike independent single agent runs where predictions diverge over time.}
    \label{fig:convergence}
\vspace{-1.0em}
\end{wrapfigure}

We are excited to support research on performance-cost scaling for test-time compute paradigms like parallel aggregation~\citep{venkatraman2026recursiveselfaggregationunlocksdeep}, multi-agent systems~\citep{tran2026singleagentllmsoutperformmultiagent}, and scaling environment interactions~\citep{shen2025thinkingvsdoingagents} on \futuresim.

\subsection{Multi agent dynamics}
\label{sec:multiagent}

As we move towards a world where agents owned by different users interact and compete~\citep{bansal2018emergent}, it is interesting to study how agents adapt in the presence of other agents. To demonstrate how \futuresim\ can support multi-agent experiments, we demonstrate a simple example, where three identical DeepSeek v3.2 agents compete simultaneously. The agents only depend on each other through an information bottleneck, which is the current aggregate prediction for each question, similar to crowd aggregates on prediction markets. In \Cref{fig:convergence} we observe that over the course of multi-agent simulations, agents converge towards similar predictions, while in independent single agent runs their predictions diverge over time. This is despite prompting agents that they will be graded on the \textit{peer score}, which \textbf{incentivizes distinct} informative predictions from the crowd aggregate. We provide details on aggregate prediction, peer score and TV distance computation in ~\Cref{app:multiagentaggtv}. The absolute agent performance itself is similar in single-agent and multi-agent runs as shown in \Cref{app:multiagentperformance}.

\looseness -1 As a simulation grounded in real-world dynamics, we hope \futuresim\ supports multi-agent research, such as improvements from self-play~\citep{silver2018general}, and studies on how diverse agents evolve~\citep{park2023generative} when equipped with richer communication channels~\citep{meta2022human}.

\section{Conclusion}
\label{sec:conclusion}

\looseness=-2 In this work, we build \futuresim\, a simulation that replays world events beyond a model's knowledge cutoff to evaluate the ability of frontier agents to adapt their predictions. \futuresim\ utilizes forecasting questions made from news documents, and updates the context available to agents at each time-step, tasking them to maintain a distribution of possible outcomes and their probabilities for each question. We find that current frontier agents act suboptimally in \futuresim\ in interesting ways, though the absolute performance we report can be considered a lower-bound on model capabilities, as improvements in the available harness, tools, or context corpus can lead to better results. That said, we show that many emerging research directions can be studied using \futuresim\ by running experiments that isolate different capabilities.

Indeed, a key advantage of our benchmark is flexibility: it exposes a minimal set of actions that allows benchmarking of any model and harness, and the data can easily be updated, or extended to other domains where chronological events are available. The only constraint is that \futuresim\ focuses on a purely predictive settings, where agent actions cannot significantly change environment dynamics, as otherwise counterfactual worlds would have to be simulated accurately. This limits applicability to decision making domains where predictions can be performative~\citep{pmlr-v119-perdomo20a}, or actions change the environment~\citep{bruce2024genie}. 
Yet, studying prediction abilities still offers valuable insights about the world model of agents, and how they adapt to new information~\citep{ha2018worldmodels}. We thus hope \futuresim\ guides AI progress towards the next frontier, of building agents that continually adapt and learn in the real-world.

\clearpage

\paragraph{Acknowledgements.} We thank Florian Brand, Albert Catalán Tatjer, Thomas Grady, \textit{hallerite}, Omar Khattab, Ilija Lichkovski, Mike Merrill, Daniel Paleka, Lluís Pastor Pérez, Yangjun Ruan, Aashay Sachdeva, Shiven Sinha, Ross Taylor, Pratyush Tiwari, Saujas Vaduguru, and Joan Velja for helpful discussions and feedback.

\bibliographystyle{colm2026_conference}
\bibliography{references}

@inproceedings{
chandak2026curating,
title={Curating the Future: A Scalable Recipe for Training Open-Ended Forecasters},
author={Nikhil Chandak and Shashwat Goel and Ameya Prabhu and Moritz Hardt and Jonas Geiping},
booktitle={ICML},
year={2026},
url={https://openreview.net/forum?id=SiMYGtHfxT}
}

@inproceedings{
paleka2026pitfalls,
title={Pitfalls in Evaluating Language Model Forecasters},
author={Daniel Paleka and Shashwat Goel and Jonas Geiping and Florian Tram{\`e}r},
booktitle={ICLR},
year={2026},
url={https://openreview.net/forum?id=z85kARAoyD}
}

@misc{chandak2025answermatchingoutperformsmultiple,
      title={Answer Matching Outperforms Multiple Choice for Language Model Evaluation}, 
      author={Nikhil Chandak and Shashwat Goel and Ameya Prabhu and Moritz Hardt and Jonas Geiping},
      year={2025},
      eprint={2507.02856},
      archivePrefix={arXiv},
      primaryClass={cs.CL},
      url={https://arxiv.org/abs/2507.02856}, 
}

@misc{mucsanyi2023proper,
  author       = {B{\'a}lint Mucs{\'a}nyi and Michael Kirchhof and Elisa Nguyen and Alexander Rubinstein and Seong Joon Oh},
  title        = {Proper/Strictly Proper Scoring Rule},
  year         = {2023},
  booktitle    = {Trustworthy Machine Learning},
  url          = {https://trustworthyml.io/},
  doi          = {10.48550/arXiv.2310.08215},
  eprint       = {2310.08215},
  archivePrefix = {arXiv}
}

@inproceedings{
yang2026llmasaprophet,
title={{LLM}-as-a-Prophet: Understanding Predictive Intelligence with Prophet Arena},
author={Qingchuan Yang and Simon Mahns and Sida Li and Anri Gu and Jibang Wu and Haifeng Xu},
booktitle={ICLR},
year={2026},
url={https://openreview.net/forum?id=VpiHkMSPqI}
}

@misc{foundation2026arcagi3newchallengefrontier,
      title={ARC-AGI-3: A New Challenge for Frontier Agentic Intelligence}, 
      author={{ARC Prize Foundation}},
      year={2026},
      eprint={2603.24621},
      archivePrefix={arXiv},
      primaryClass={cs.AI},
      url={https://arxiv.org/abs/2603.24621}, 
}

@inproceedings{
sinha2026the,
title={The Illusion of Diminishing Returns: Measuring Long Horizon Execution in {LLM}s},
author={Akshit Sinha and Arvindh Arun and Shashwat Goel and Steffen Staab and Jonas Geiping},
booktitle={The Fourteenth International Conference on Learning Representations},
year={2026},
url={https://openreview.net/forum?id=3lm8lWYxiq}
}

@misc{backlund2025vendingbenchbenchmarklongtermcoherence,
      title={Vending-Bench: A Benchmark for Long-Term Coherence of Autonomous Agents}, 
      author={Axel Backlund and others},
      year={2025},
      eprint={2502.15840},
      archivePrefix={arXiv},
      primaryClass={cs.AI},
      url={https://arxiv.org/abs/2502.15840}, 
}

@article{atanasov2020small,
  title   = {Small Steps to Accuracy: Incremental Belief Updaters Are Better Forecasters},
  author  = {Atanasov, Pavel and Witkowski, Jens and Ungar, Lyle and Mellers, Barbara and Tetlock, Philip},
  journal = {Organizational Behavior and Human Decision Processes},
  volume  = {160},
  pages   = {19--35},
  year    = {2020},
  doi     = {10.1016/j.obhdp.2020.02.001}
}

@article{mellers2015identifying,
  title   = {Identifying and Cultivating Superforecasters as a Method of Improving Probabilistic Predictions},
  author  = {Mellers, Barbara and Stone, Eric and Murray, Terry and Minster, Angela and Rohrbaugh, Nick and Bishop, Michael and Chen, Eva and Baker, Joshua and Hou, Yuan and Horowitz, Michael and Ungar, Lyle and Tetlock, Philip},
  journal = {Perspectives on Psychological Science},
  volume  = {10},
  number  = {3},
  pages   = {267--281},
  year    = {2015},
  doi     = {10.1177/1745691615577794}
}

@misc{zhang2026recursivelanguagemodels,
      title={Recursive Language Models}, 
      author={Alex L. Zhang and Tim Kraska and Omar Khattab},
      year={2026},
      eprint={2512.24601},
      archivePrefix={arXiv},
      primaryClass={cs.AI},
      url={https://arxiv.org/abs/2512.24601}, 
}

@inproceedings{
hardt2024testtime,
title={Test-Time Training on Nearest Neighbors for Large Language Models},
author={Moritz Hardt and Yu Sun},
booktitle={The Twelfth International Conference on Learning Representations},
year={2024},
url={https://openreview.net/forum?id=CNL2bku4ra}
}

@inproceedings{
yang2025synthetic,
title={Synthetic continued pretraining},
author={Zitong Yang and Neil Band and Shuangping Li and Emmanuel Candes and Tatsunori Hashimoto},
booktitle={The Thirteenth International Conference on Learning Representations},
year={2025},
url={https://openreview.net/forum?id=07yvxWDSla}
}

@inproceedings{
wang2025openhands,
title={OpenHands: An Open Platform for {AI} Software Developers as Generalist Agents},
author={Xingyao Wang and Boxuan Li and Yufan Song and Frank F. Xu and Xiangru Tang and Mingchen Zhuge and Jiayi Pan and Yueqi Song and Bowen Li and Jaskirat Singh and Hoang H. Tran and Fuqiang Li and Ren Ma and Mingzhang Zheng and Bill Qian and Yanjun Shao and Niklas Muennighoff and Yizhe Zhang and Binyuan Hui and Junyang Lin and Robert Brennan and Hao Peng and Heng Ji and Graham Neubig},
booktitle={The Thirteenth International Conference on Learning Representations},
year={2025},
url={https://openreview.net/forum?id=OJd3ayDDoF}
}

@misc{zhang2026hyperagents,
      title={Hyperagents}, 
      author={Jenny Zhang and Bingchen Zhao and Wannan Yang and Jakob Foerster and Jeff Clune and Minqi Jiang and Sam Devlin and Tatiana Shavrina},
      year={2026},
      eprint={2603.19461},
      archivePrefix={arXiv},
      primaryClass={cs.AI},
      url={https://arxiv.org/abs/2603.19461}, 
}

@misc{lou2026autoharnessimprovingllmagents,
      title={AutoHarness: improving LLM agents by automatically synthesizing a code harness}, 
      author={Xinghua Lou and Miguel Lázaro-Gredilla and Antoine Dedieu and Carter Wendelken and Wolfgang Lehrach and Kevin P. Murphy},
      year={2026},
      eprint={2603.03329},
      archivePrefix={arXiv},
      primaryClass={cs.CL},
      url={https://arxiv.org/abs/2603.03329}, 
}

@article{shao2025reasonir,
      title={ReasonIR: Training Retrievers for Reasoning Tasks}, 
      author={Rulin Shao and Rui Qiao and Varsha Kishore and Niklas Muennighoff and Xi Victoria Lin and Daniela Rus and Bryan Kian Hsiang Low and Sewon Min and Wen-tau Yih and Pang Wei Koh and Luke Zettlemoyer},
      year={2025},
      journal={arXiv preprint arXiv:2504.20595},
      url={https://arxiv.org/abs/2504.20595}, 
}

@misc{khattab2020colbertefficienteffectivepassage,
      title={ColBERT: Efficient and Effective Passage Search via Contextualized Late Interaction over BERT}, 
      author={Omar Khattab and Matei Zaharia},
      year={2020},
      eprint={2004.12832},
      archivePrefix={arXiv},
      primaryClass={cs.IR},
      url={https://arxiv.org/abs/2004.12832}, 
}

@misc{murphy2026agenticforecastingusingsequential,
      title={Agentic Forecasting using Sequential Bayesian Updating of Linguistic Beliefs}, 
      author={Kevin Murphy},
      year={2026},
      eprint={2604.18576},
      archivePrefix={arXiv},
      primaryClass={cs.AI},
      url={https://arxiv.org/abs/2604.18576}, 
}

@misc{silver2025eraexperience,
  title        = {Welcome to the Era of Experience},
  author       = {Silver, David and Sutton, Richard S.},
  year         = {2025},
  note         = {Preprint of a chapter to appear in \emph{Designing an Intelligence}, MIT Press},
  url          = {https://storage.googleapis.com/deepmind-media/Era-of-Experience%20/The%20Era%20of%20Experience%20Paper.pdf}
}

@inproceedings{
snell2025scaling,
title={Scaling {LLM} Test-Time Compute Optimally Can be More Effective than Scaling Parameters for Reasoning},
author={Charlie Victor Snell and Jaehoon Lee and Kelvin Xu and Aviral Kumar},
booktitle={The Thirteenth International Conference on Learning Representations},
year={2025},
url={https://openreview.net/forum?id=4FWAwZtd2n}
}

@misc{venkatraman2026recursiveselfaggregationunlocksdeep,
      title={Recursive Self-Aggregation Unlocks Deep Thinking in Large Language Models}, 
      author={Siddarth Venkatraman and Vineet Jain and Sarthak Mittal and Vedant Shah and Johan Obando-Ceron and Yoshua Bengio and Brian R. Bartoldson and Bhavya Kailkhura and Guillaume Lajoie and Glen Berseth and Nikolay Malkin and Moksh Jain},
      year={2026},
      eprint={2509.26626},
      archivePrefix={arXiv},
      primaryClass={cs.LG},
      url={https://arxiv.org/abs/2509.26626}, 
}

@misc{tran2026singleagentllmsoutperformmultiagent,
      title={Single-Agent LLMs Outperform Multi-Agent Systems on Multi-Hop Reasoning Under Equal Thinking Token Budgets}, 
      author={Dat Tran and Douwe Kiela},
      year={2026},
      eprint={2604.02460},
      archivePrefix={arXiv},
      primaryClass={cs.CL},
      url={https://arxiv.org/abs/2604.02460}, 
}

@misc{shen2025thinkingvsdoingagents,
      title={Thinking vs. Doing: Agents that Reason by Scaling Test-Time Interaction}, 
      author={Junhong Shen and Hao Bai and Lunjun Zhang and Yifei Zhou and Amrith Setlur and Shengbang Tong and Diego Caples and Nan Jiang and Tong Zhang and Ameet Talwalkar and Aviral Kumar},
      year={2025},
      eprint={2506.07976},
      archivePrefix={arXiv},
      primaryClass={cs.LG},
      url={https://arxiv.org/abs/2506.07976}, 
}

@article{silver2018general,
  title={A general reinforcement learning algorithm that masters chess, shogi, and Go through self-play},
  author={Silver, David and Hubert, Thomas and Schrittwieser, Julian and Antonoglou, Ioannis and Lai, Matthew and Guez, Arthur and Lanctot, Marc and Sifre, Laurent and Kumaran, Dharshan and Graepel, Thore and others},
  journal={Science},
  volume={362},
  number={6419},
  pages={1140--1144},
  year={2018},
  publisher={American Association for the Advancement of Science}
}

@inproceedings{
bansal2018emergent,
title={Emergent Complexity via Multi-Agent Competition},
author={Trapit Bansal and Jakub Pachocki and Szymon Sidor and Ilya Sutskever and Igor Mordatch},
booktitle={International Conference on Learning Representations},
year={2018},
url={https://openreview.net/forum?id=Sy0GnUxCb},
}

@inproceedings{park2023generative,
  title={Generative agents: Interactive simulacra of human behavior},
  author={Park, Joon Sung and O'Brien, Joseph C. and Cai, Carrie J. and Morris, Meredith Ringel and Liang, Percy and Bernstein, Michael S.},
  booktitle={Proceedings of the 36th Annual ACM Symposium on User Interface Software and Technology},
  year={2023}
}

@article{meta2022human,
  title={Human-level play in the game of diplomacy by combining language models with strategic reasoning},
  author={Brown, Noam and Bakhtin, Anton and Dinan, Emily and Farina, Gabriele and Flaherty, Colin and Fried, Daniel and Goff, Andrew and Gray, Jonathan and Hu, Hengyuan and others},
  journal={Science},
  volume={378},
  number={6624},
  pages={1067--1074},
  year={2022},
  publisher={American Association for the Advancement of Science},
  note={Meta Fundamental AI Research Diplomacy Team (FAIR)}
}

@inproceedings{
paglieri2025balrog,
title={{BALROG}: Benchmarking Agentic {LLM} and {VLM} Reasoning On Games},
author={Davide Paglieri and Bart{\l}omiej Cupia{\l} and Samuel Coward and Ulyana Piterbarg and Maciej Wolczyk and Akbir Khan and Eduardo Pignatelli and {\L}ukasz Kuci{\'n}ski and Lerrel Pinto and Rob Fergus and Jakob Nicolaus Foerster and Jack Parker-Holder and Tim Rockt{\"a}schel},
booktitle={The Thirteenth International Conference on Learning Representations},
year={2025},
url={https://openreview.net/forum?id=fp6t3F669F}
}

@misc{he2026textttycbenchbenchmarkingaiagents,
      title={$\texttt{YC-Bench}$: Benchmarking AI Agents for Long-Term Planning and Consistent Execution}, 
      author={Muyu He and Adit Jain and Anand Kumar and Vincent Tu and Soumyadeep Bakshi and Sachin Patro and Nazneen Rajani},
      year={2026},
      eprint={2604.01212},
      archivePrefix={arXiv},
      primaryClass={cs.CL},
      url={https://arxiv.org/abs/2604.01212}, 
}

@misc{seshadri2026lostsimulationllmsimulatedusers,
      title={Lost in Simulation: LLM-Simulated Users are Unreliable Proxies for Human Users in Agentic Evaluations}, 
      author={Preethi Seshadri and Samuel Cahyawijaya and Ayomide Odumakinde and Sameer Singh and Seraphina Goldfarb-Tarrant},
      year={2026},
      eprint={2601.17087},
      archivePrefix={arXiv},
      primaryClass={cs.HC},
      url={https://arxiv.org/abs/2601.17087}, 
}

@misc{anthropic2025projectvend,
  author       = {{Anthropic}},
  title        = {{Project Vend: Can Claude Run a Small Shop? (And Why Does That Matter?)}},
  year         = {2025},
  month        = jun,
  day          = {27},
  howpublished = {\url{https://www.anthropic.com/research/project-vend-1}},
  note         = {Accessed: 2026-05-03}
}

@misc{halawi2024approachinghumanlevelforecastinglanguage,
      title={Approaching Human-Level Forecasting with Language Models}, 
      author={Danny Halawi and Fred Zhang and Chen Yueh-Han and Jacob Steinhardt},
      year={2024},
      eprint={2402.18563},
      archivePrefix={arXiv},
      primaryClass={cs.LG},
      url={https://arxiv.org/abs/2402.18563}, 
}

@misc{zhang2026predictionarenabenchmarkingai,
      title={Prediction Arena: Benchmarking AI Models on Real-World Prediction Markets}, 
      author={Jaden Zhang and Gardenia Liu and Oliver Johansson and Hileamlak Yitayew and Kamryn Ohly and Grace Li},
      year={2026},
      eprint={2604.07355},
      archivePrefix={arXiv},
      primaryClass={cs.LG},
      url={https://arxiv.org/abs/2604.07355}, 
}

@misc{grady2026kellybenchbenchmarklonghorizonsequential,
      title={KellyBench: A Benchmark for Long-Horizon Sequential Decision Making}, 
      author={Thomas Grady and Kip Parker and Iliyan Zarov and Henry Course and Chengxi Taylor and Ross Taylor},
      year={2026},
      eprint={2604.27865},
      archivePrefix={arXiv},
      primaryClass={cs.AI},
      url={https://arxiv.org/abs/2604.27865}, 
}

@article{ha2018worldmodels,
  doi = {10.5281/ZENODO.1207631},
  url = {https://zenodo.org/record/1207631},
  author = {Ha, David and Schmidhuber, Jürgen},
  title = {World Models},
  publisher = {Zenodo},
  year = {2018},
  copyright = {Creative Commons Attribution 4.0}
}

@inproceedings{
froger2026gaia,
title={Gaia2: Benchmarking {LLM} Agents on Dynamic and  Asynchronous Environments},
author={Romain Froger and Pierre Andrews and Matteo Bettini and Amar Budhiraja and Ricardo Silveira Cabral and Virginie Do and Emilien Garreau and Jean-Baptiste Gaya and Hugo Lauren{\c{c}}on and Maxime Lecanu and Kunal Malkan and Dheeraj Mekala and Pierre Menard and Gerard Moreno-Torres Bertran and Ulyana Piterbarg and Mikhail Plekhanov and Mathieu Rita and Andrey Rusakov and Vladislav Vorotilov and Mengjue Wang and Ian Yu and Amine Benhalloum and Gr{\'e}goire Mialon and Thomas Scialom},
booktitle={The Fourteenth International Conference on Learning Representations},
year={2026},
url={https://openreview.net/forum?id=9gw03JpKK4}
}

@inproceedings{
karger2025forecastbench,
title={ForecastBench: A Dynamic Benchmark of {AI} Forecasting Capabilities},
author={Ezra Karger and Houtan Bastani and Chen Yueh-Han and Zachary Jacobs and Danny Halawi and Fred Zhang and Philip Tetlock},
booktitle={The Thirteenth International Conference on Learning Representations},
year={2025},
url={https://openreview.net/forum?id=lfPkGWXLLf}
}

@article{WEBBY199691,
title = {Judgemental and statistical time series forecasting: a review of the literature},
journal = {International Journal of Forecasting},
volume = {12},
number = {1},
pages = {91-118},
year = {1996},
url = {https://www.sciencedirect.com/science/article/pii/0169207095006443},
author = {Richard Webby and Marcus O'Connor},
}

@misc{thai2026sweevobenchmarkingcodingagents,
      title={SWE-EVO: Benchmarking Coding Agents in Long-Horizon Software Evolution Scenarios}, 
      author={Minh V. T. Thai and Tue Le and Dung Nguyen Manh and Huy Phan Nhat and Nghi D. Q. Bui},
      year={2026},
      eprint={2512.18470},
      archivePrefix={arXiv},
      primaryClass={cs.SE},
      url={https://arxiv.org/abs/2512.18470}, 
}

@misc{shi2026tauknowledgeevaluatingconversationalagents,
      title={$\tau$-Knowledge: Evaluating Conversational Agents over Unstructured Knowledge}, 
      author={Quan Shi and Alexandra Zytek and Pedram Razavi and Karthik Narasimhan and Victor Barres},
      year={2026},
      eprint={2603.04370},
      archivePrefix={arXiv},
      primaryClass={cs.AI},
      url={https://arxiv.org/abs/2603.04370}, 
}

@misc{farquhar2019robustevaluationscontinuallearning,
      title={Towards Robust Evaluations of Continual Learning}, 
      author={Sebastian Farquhar and Yarin Gal},
      year={2019},
      eprint={1805.09733},
      archivePrefix={arXiv},
      primaryClass={stat.ML},
      url={https://arxiv.org/abs/1805.09733}, 
}

@inproceedings{bruce2024genie,
  title={Genie: Generative interactive environments},
  author={Bruce, Jake and Dennis, Michael D and Edwards, Ashley and Parker-Holder, Jack and Shi, Yuge and Hughes, Edward and Lai, Matthew and Mavalankar, Aditi and Steigerwald, Richie and Apps, Chris and others},
  booktitle={Forty-first International Conference on Machine Learning},
  year={2024}
}

@inproceedings{yang-etal-2018-hotpotqa,
  title = "{H}otpot{QA}: A Dataset for Diverse, Explainable Multi-hop Question Answering",
  author = "Yang, Zhilin and
    Qi, Peng and
    Zhang, Saizheng and
    Bengio, Yoshua and
    Cohen, William and
    Salakhutdinov, Ruslan and
    Manning, Christopher D.",
  booktitle = "EMNLP",
  year = "2018",
  url = "https://aclanthology.org/D18-1259/",
}

@misc{wei2025browsecomp,
  title = {BrowseComp: A Simple Yet Challenging Benchmark for Browsing Agents},
  author = {Wei, Jason and
    Sun, Zhiqing and
    Papay, Spencer and
    McKinney, Scott and
    Han, Jeffrey and
    Fulford, Isa and
    Chung, Hyung Won and
    Passos, Alex Tachard and
    Fedus, William and
    Glaese, Amelia},
  year = {2025},
  eprint = {2504.12516},
  archivePrefix = {arXiv},
  primaryClass = {cs.CL},
  doi = {10.48550/arXiv.2504.12516},
  url = {https://arxiv.org/abs/2504.12516}
}

@inproceedings{jin2025searchr1,
  title = {Search-{R1}: Training {LLM}s to Reason and Leverage Search Engines with Reinforcement Learning},
  author = {Jin, Bowen and
    Zeng, Hansi and
    Yue, Zhenrui and
    Yoon, Jinsung and
    Ar{\i}k, Sercan O. and
    Wang, Dong and
    Zamani, Hamed and
    Han, Jiawei},
  booktitle = {Proceedings of the 2nd Conference on Language Modeling},
  year = {2025},
  address = {Montreal, Canada},
  url = {https://openreview.net/forum?id=Rwhi91ideu}
}

@inproceedings{
damani2026beyond,
title={Beyond Binary Rewards: Training {LM}s to Reason About Their Uncertainty},
author={Mehul Damani and Isha Puri and Stewart Slocum and Idan Shenfeld and Leshem Choshen and Yoon Kim and Jacob Andreas},
booktitle={The Fourteenth International Conference on Learning Representations},
year={2026},
url={https://openreview.net/forum?id=ASQ649zdHm}
}

@online{nagel2016-ccnews,
  author       = {Sebastian Nagel},
  title        = {Common Crawl News Dataset},
  year         = {2016},
  url          = {https://data.commoncrawl.org/crawl-data/CC-NEWS/index.html},
  urldate      = {2025-09-01},
  organization = {Common Crawl}
}

@inproceedings{
lazaridou2021mind,
title={Mind the Gap: Assessing Temporal Generalization in Neural Language Models},
author={Angeliki Lazaridou and Adhiguna Kuncoro and Elena Gribovskaya and Devang Agrawal and Adam Liska and Tayfun Terzi and Mai Gimenez and Cyprien de Masson d'Autume and Tom{\'a}{\v{s}} Ko{\v{c}}isk{\'y} and Sebastian Ruder and Dani Yogatama and Kris Cao and Susannah Young and Phil Blunsom},
booktitle={Advances in Neural Information Processing Systems},
editor={A. Beygelzimer and Y. Dauphin and P. Liang and J. Wortman Vaughan},
year={2021},
url={https://openreview.net/forum?id=73OmmrCfSyy}
}

@inproceedings{yao2022wildtime,
 author = {Yao, Huaxiu and Choi, Caroline and Cao, Bochuan and Lee, Yoonho and Koh, Pang Wei and Finn, Chelsea},
 booktitle = {Advances in Neural Information Processing Systems},
 editor = {S. Koyejo and S. Mohamed and A. Agarwal and D. Belgrave and K. Cho and A. Oh},
 pages = {10309--10324},
 publisher = {Curran Associates, Inc.},
 title = {Wild-Time: A Benchmark of in-the-Wild Distribution Shift over Time},
 url = {https://proceedings.neurips.cc/paper_files/paper/2022/file/43119db5d59f07cc08fca7ba6820179a-Paper-Datasets_and_Benchmarks.pdf},
 volume = {35},
 year = {2022}
}

@InProceedings{pmlr-v119-perdomo20a,
  title = 	 {Performative Prediction},
  author =       {Perdomo, Juan and Zrnic, Tijana and Mendler-D{\"u}nner, Celestine and Hardt, Moritz},
  booktitle = 	 {Proceedings of the 37th International Conference on Machine Learning},
  pages = 	 {7599--7609},
  year = 	 {2020},
  editor = 	 {III, Hal Daumé and Singh, Aarti},
  volume = 	 {119},
  series = 	 {Proceedings of Machine Learning Research},
  month = 	 {13--18 Jul},
  publisher =    {PMLR},
  url = 	 {https://proceedings.mlr.press/v119/perdomo20a.html},
}

@misc{zhou2026mindsim2realgapuser,
      title={Mind the Sim2Real Gap in User Simulation for Agentic Tasks}, 
      author={Xuhui Zhou and Weiwei Sun and Qianou Ma and Yiqing Xie and Jiarui Liu and Weihua Du and Sean Welleck and Yiming Yang and Graham Neubig and Sherry Tongshuang Wu and Maarten Sap},
      year={2026},
      eprint={2603.11245},
      archivePrefix={arXiv},
      primaryClass={cs.AI},
      url={https://arxiv.org/abs/2603.11245}, 
}

@misc{vu2023freshllms,
      title={FreshLLMs: Refreshing Large Language Models with Search Engine Augmentation}, 
      author={Tu Vu and Mohit Iyyer and Xuezhi Wang and Noah Constant and Jerry Wei and Jason Wei and Chris Tar and Yun-Hsuan Sung and Denny Zhou and Quoc Le and Thang Luong},
      year={2023},
      eprint={2310.03214},
      archivePrefix={arXiv},
      primaryClass={cs.CL}
}
\clearpage
\appendix
\etocdepthtag.toc{appendix}
\section*{Appendix}
\crefalias{section}{appendix}
\crefalias{subsection}{appendix}
\crefalias{subsubsection}{appendix}

\etocsettagdepth{mainmatter}{none}
\etocsettagdepth{appendix}{subsection}
\etocsetnexttocdepth{subsection}
\etocsettocstyle{\subsection*{Contents}}{}
\tableofcontents

\clearpage

\crefalias{section}{appendix}
\crefalias{section}{appendix}
\section{Data}
\label{app:questiondetails}

\subsection{Question creation methodology}

\looseness -1 A natural way to obtain resolved forecasting questions is to reuse questions from prediction markets~\citep{yang2026llmasaprophet}. While providing useful grounding to questions humans made predictions on, recent work highlights limitations of using them to study forecasting~\citep{chandak2026curating, paleka2026pitfalls}: they cover a limited number of events, focus on sports, cryptocurrency, US politics, and entertainment~\citep{paleka2026pitfalls}, and mostly use binary yes/no or multiple-choice formats. In \futuresim\ , we instead take the more scalable approach of synthesizing short-answer forecasting questions from any (news) source document, introduced in ~\citet{chandak2026curating}. We also think this can help explore weight finetuning-based continual learning approaches on our benchmark, which might need more data to show advantages over inference-only counterparts.

Following~\cite{chandak2026curating}, the question generation pipeline takes timestamped articles as source documents and prompts an LLM to generate free-form short-answer questions whose answers are contained in the source articles. The question title, background, and resolution criteria are written as if the question were being asked before the answer is known. The pipeline then applies leakage checks, validates the question format, extracts and revises dates, filters invalid answer types, and removes questions that fail quality checks.

\paragraph{Our Data.} The current question set is generated from Al~Jazeera articles in the first quarter of 2026. We found this to be the highest quality source, which is freely accessible in the CCNews corpus, which covers a wide range of global events. We initiated our question generation from 10,000+ source articles, eventually narrowing down to just 330 high-quality questions ($3\%$).

\phantomsection
\label{app:sample-question}
\begin{tcolorbox}[
  colback=blue!5,
  colframe=blue!25,
  colbacktitle=blue!25,
  coltitle=black,
  boxrule=0.5pt,
  arc=1pt,
  left=4pt,
  right=4pt,
  top=3pt,
  bottom=3pt,
  title={Sample Forecasting Question},
  fonttitle=\bfseries,
  fontupper=\small,
]
\textbf{Question.} Who will be sworn in as Nepal's new prime minister?

\smallskip
\textbf{Resolution Date:} March 6, 2026

\textbf{Answer Type.} String (Name)

\textbf{Ground-Truth Answer.} Balendra Shah

\textbf{Source.} \href{https://www.aljazeera.com/features/2026/3/27/now-in-power-nepals-rapper-politician-balen-shah-faces-new-challenge}{Al~Jazeera}
\end{tcolorbox}

\paragraph{Additional Refinement.} The key changes from the pipeline in \cite{chandak2026curating} to improve data quality are as follows:
\begin{itemize}
    \item We create a test set resolving from January 1 to March 28, 2026, so the benchmark targets events after the knowledge cutoffs of the frontier models we evaluate.
    \item We revise each resolution date to the earliest date on which the answer could be inferred with certainty from public evidence. Our pipeline prompts models like GPT 5.5 with web-search to answer the question and identify the earliest inferable date. We remove questions where the model's answer is correct, and the inferred resolution date is before January 1, 2026. We provide the prompt in \Cref{app:prompt-late-resolve}.
    \item We remove questions that models can answer confidently using search capped to June 2025, as they are too stale or easy. We also remove questions that models still fail to answer even with full web search access as of April 2026, as this might be due to label noise. This approach makes our evaluation quite distinct from past factuality evaluations, which have been used to measure temporal adaptation~\cite{vu2023freshllms}.
\end{itemize}

\subsection{Example of a Discarded Question}
\label{app:discarded-question-example}

\begin{promptframe}{Discarded Forecasting Question}
\textbf{Question.} Which country will PAHO review alongside the United States at
its virtual meeting on measles elimination status on April 13, 2026?

\medskip
\textbf{Background.} In the Americas, countries can lose measles elimination
status if local transmission continues for too long. A regional health body is
monitoring countries whose status may be at risk.
\end{promptframe}

\begin{promptframe}{Response of GPT-5.5 with web search}
\textbf{Answer.} Mexico.

\medskip
\textbf{Earliest inferable date.} November 11, 2025.

\medskip
The official PAHO announcement on January 16, 2026 explicitly named the United
States and Mexico, but the answer was already strongly determined by public
evidence on November 11, 2025. By then, Canada had just lost its measles-free
status, PAHO's director had publicly discussed which countries were next, and
the United States and Mexico were the two remaining countries approaching the
12-month local-transmission threshold. The model concluded that the January
announcement was confirmation rather than the first date on which the answer
could be inferred.
\end{promptframe}

The example above illustrates why updating resolution dates is important, rather than using only the source article's date. At first sight, the question looks
like a valid April 2026 forecasting question: the official PAHO meeting was
announced for April 13, 2026, and the answer appeared explicitly in a January
2026 PAHO notice. However, a model with web search judged that the answer was
already inferable before our evaluation window.

\textbf{Final set of tasks.} Each retained question becomes one forecasting task in our environment. In the main experiments, all tasks are visible from the start of the simulation on December 24, 2025, and remain active until their resolution date. Agents may submit and revise forecasts while a question is active. Each submission is a probability distribution over up to five free-form outcomes.

\subsection{Question Distribution}

\paragraph{Topic Distribution.} Table~\ref{app:topic-distribution} summarizes the topic distribution of the 330 forecasting questions in \futuresim.

\begin{table}[t]
\centering
\small
\begin{tabular}{lrr}
\toprule
Topic & \# Questions & Fraction (\%) \\
\midrule
International Politics \& Diplomacy & 91 & 27.6 \\
Conflict \& War & 78 & 23.6 \\
Sports & 75 & 22.7 \\
Crime \& Justice & 26 & 7.9 \\
US Politics & 20 & 6.1 \\
Disasters \& Accidents & 14 & 4.2 \\
Other & 26 & 7.9 \\
\midrule
Total & 330 & 100 \\
\bottomrule
\end{tabular}
\vspace{0.2cm}
\caption{Topic distribution of the 330 forecasting questions in \futuresim. ``Other'' bundles the long tail of low-frequency topics (Technology \& AI, Entertainment \& Culture, Business \& Economy, Health \& Medicine, Climate \& Environment, Religion, Science \& Space).}
\label{app:topic-distribution}
\end{table}

\paragraph{Resolution distribution.} The questions resolve over 84 distinct calendar dates between January 1 and March 28, 2026. Of the 330 questions, 122 resolve in January, 92 in February, and 116 in March. 

\paragraph{Empirical difficulty.} We also summarize the difficulty of our question set by how often the final top prediction matches the resolved answer across the top four models, in our harness. By this measure, 229 questions are answered correctly by no run, 35 by exactly one run, 17 by exactly two runs, 16 by exactly three runs, and 33 by all four runs. Equivalently, 101 questions are solved by at least one run.

\subsection{Mitigating Leakage of Future Information}

Overall, we apply the following mitigations to prevent leakage of future information before it is available to agents in the simulation:

\paragraph{Tasks.} The filtering pipeline to ensure the created forecasting questions don't leak future information, and are not solvable before their resolution date is described in the previous section.

\paragraph{Context.} We sandbox agents, so they can only read articles from a folder maintained by the environment code, which adds only articles from the current date into the folder each time \texttt{next\_day()} is called by an agent. At the start of the simulation, all articles before the simulation's first date are stored in this folder. We do not provide the agent access to live web search. We provide implementation details of the sandbox in \Cref{app:sandboxdetails}.

Consistent with prior work, our attempts to use date filtering functionality in popular search APIs like Brave failed, with future information leaking through both updates to the articles and ranking indices~\citep{paleka2026pitfalls}. We do set up a local context repository of news articles from CCNews. CCNews captures an offline snapshot of the news articles, with the date it was taken. The hybrid search tool we provide to agents is based on LanceDB, which enables optimized date-range-controlled search over the articles. Inside it, we use the Qwen3 Embedding 8B model for semantic search, which finished training in mid-2025 and thus does not have access to recent information that could influence similarity scores and, consequently, retrieval rankings.

\crefalias{section}{appendix}
\section{Environment Details}
\label{app:environment-details}

\subsection{Our baseline harness}
\label{app:harness}
The motivation for our custom harness is that long-horizon forecasting tests agents on several frontiers at once: they must repeatedly revisit a large state, search a growing corpus, remember partially useful evidence over many days, and still leave enough context budget for reasoning and forecast submission. A minimal ``just give the model shell and tool access'' setup leaves ample performance on the table for open-weights models, especially, so our baseline harness adds lightweight structure intended to improve reliability without solving the task for the agent. In particular, it incorporates the following features:
\begin{enumerate}
    \item \textbf{Context consumption feedback:} After each tool call, the agent receives feedback about remaining context budget and approximate context occupancy. This is useful because the task spans thousands of turns, and without explicit budget awareness, agents often spend a lot of context browsing or performing repeated file reads, leaving too little room for final reasoning and forecast submission. We also clear the live interaction context each day and explicitly inform the agent about this, encouraging it to externalize useful information rather than assuming that important evidence will remain in-window indefinitely.
    \item \textbf{Structured memory tools:} Instead of asking the agent to maintain free-form notes arbitrarily in its workspace, we expose external memory through explicit tool calls with named entries and bounded fields. The goal is to make memory writing and retrieval deliberate actions rather than accidental byproducts of shell usage. This structure also makes it easier for the agent to store compact, queryable summaries of evidence, lessons from resolved questions, and reusable forecasting heuristics.
    \item \textbf{Per-question memory:} In addition to global notes, the harness maintains memory entries attached to individual questions. This is motivated by the fact that forecasting requires a mix of cross-question lessons and question-specific evidence: a general lesson about overconfidence should be stored differently from a candidate list or event-specific rationale for one market. Per-question memory helps agents revisit an active question after many simulated days without having to reconstruct all prior reasoning from scratch.
    \item \textbf{Forced memory phase:} When the agent ends a day, or when the context budget becomes too tight, the harness enters an explicit memory-update phase before actually advancing. During this phase, the agent is encouraged to compress what it learned into persistent notes and leave a cleaner state for the next day. The motivation is to prevent a common failure mode in long runs where agents defer summarization until too late, lose transient evidence to compaction, and then repeatedly rediscover the same information.
    \item \textbf{Procedural forecasting scaffolding:} The prompt encourages a concrete workflow: inspect the active questions, prioritize the ones most worth updating, search for relevant evidence, submit forecasts, update memory, and only then proceed to the next day. This scaffolding is intentionally lightweight: it does not tell the model what the answer is, but it does reduce dithering and helps it allocate attention to imminent resolutions, stale forecasts, and questions where new evidence is most likely to matter. In our experience, this kind of process guidance is especially important for getting models to repeatedly revise forecasts rather than treating an early answer as final.
\end{enumerate}

\subsection{Simulation logic} 
The simulation progresses in discrete time steps. At the beginning of the simulation, the agents are initialized with a prompt specifying the task, the scoring rules, the available tools and context, and the active questions; the native harness prompt is provided in \Cref{app:prompt-native-harness}, with additional prompt variants in \Cref{app:prompts}. At each timestep, the agent can observe and interact with the environment via MCP tool calls, following the OpenReward Standard: \url{https://openrewardstandard.io/}. The agents can observe the currently active questions, search over the latest news corpus, and make predictions. When the agent wants to proceed to the next timestep, it can call the \texttt{next\_day} action, which advances the environment by one timestep and updates the state accordingly. The state update reflects the new simulation date, the number of newly available articles, the number of active questions, and any questions that resolved on that date. For each resolved prediction, the environment exposes the ground-truth outcome and the score the agent accrued, in chronological order, which the agents can use to calibrate later predictions and learn new strategies. The simulation terminates after the last question resolves, at which point the final metrics are computed.

\subsection{Sandboxing agents to avoid contamination} 
\label{app:sandboxdetails}
To ensure agents cannot access future information, we sandbox them carefully using \texttt{bwrap} on a Linux server. Each harness runs inside a sandbox, an isolated process with its own private filesystem view and with controlled network access, to prevent leakage of unintended post-cutoff information. We ensure (i) \emph{No live web search}: the harness has no direct network access, and the only access to external links is LLM provider endpoints used to run the model. \texttt{WebSearch}, \texttt{WebFetch}, and any commands like \texttt{curl} to any other endpoints are blocked. (ii) \emph{Date-gated article corpus}: only the \texttt{articles/YYYY/MM/DD/} directories up to the current simulation date are exposed inside the sandbox, and the \texttt{search\_news} MCP tool automatically caps its \texttt{to\_date} so the underlying index cannot return future-dated articles. (iii) \emph{Read-only environment state}: \texttt{state.csv} containing the questions, their metadata, and the agent's forecasts is read-only, and the environment codebase, dataset, and other run directories are not visible from inside the sandbox at all. The agent has read-write access only to its own sandboxed workspace.

\subsection{Cost}
\label{app:cost}

Our experiments use a mix of direct API usage for querying the models and a local retrieval infrastructure. As frontier models like GPT 5.5 and Opus 4.6 are very expensive to run for long evaluations, we evaluated their agents through their providers' coding plans, each costing roughly $\$220$ per month. We also use a similar coding plan for GLM 5.1. These prices are not directly comparable to token-metered API calls, but we believe the equivalent cost of running the same long-horizon experiments through the API would be substantially higher, potentially on the order of $10\times$ more for GPT 5.5 and Opus 4.6. For models run through metered API inference, DeepSeek V4 Pro costs about $\$50$ per simulation run, while Qwen3.6 Plus through OpenRouter costs about $\$150$ per run. In addition, each run requires answer matching over many candidate outcomes, for which we use DeepSeek V3.2 for this step as it is cheap and reliable for answer matching~\citep{chandak2025answermatchingoutperformsmultiple}, costing less than $\$50$ for over 10{,}000 queries in a full run. Finally, our retrieval setup requires hosting Qwen3 Embedding 8B, used by the LanceDB search pipeline, on a single A100 or H100 GPU. These infrastructure costs are modest relative to the main agent runs for closed models, but they are part of the full end-to-end cost of reproducing our experiments.

\subsection{Context Details: CCNews}
\label{app:contextdetails}

The external context in \futuresim\ is a dated snapshot of Common Crawl News (CCNews), which lets us replay what evidence was available on each simulated day without relying on live web access. We first convert the raw crawl into a deduplicated article corpus, preserving article text, source, URL, and publication date. At runtime, agents interact only with the portion of this corpus whose dates are on or before the current simulation day.

In our experiments, agents access the corpus through a local LanceDB-backed retrieval tool rather than unrestricted browsing. Articles are split into 512-token text chunks, embedded with Qwen3 Embedding 8B, and indexed for hybrid retrieval that combines semantic search and keyword matching. The search interface also supports explicit date bounds, so both the folder with terminal command access and the retrieval tool have the same temporal cutoff. We treat this context store as part of the environment rather than part of any particular agent design, making \futuresim\ compatible with different retrieval, context-management, and memory strategies as long as they respect the date restrictions.

\paragraph{Leakage issues with Brave Search API.}
Search APIs such as Brave can be useful for retrieving news and other evidence, but we found that their date filters are not sufficient for leakage-free simulation. In one audit case, a question from our test-set asked: ``Which country will the silver medalist in the women’s downhill alpine skiing event at the Milan-Cortina Winter Olympics represent by 8 February 2026?'' The question's resolution date was February 7, 2026 with ground truth Germany. In one run of our simulation using Brave API, we found that on simulated January 30, 2026, the agent queried Brave for 2026 Winter Olympics women's downhill favorites and silver-medal country with \texttt{from\_date=2026-01-01} and \texttt{to\_date=2026-01-30}. Brave returned an Olympics Wiki/Fandom result dated January 9, 2026 whose snippet already said the event was held on February 8 and that Breezy Johnson/United States won gold, Emma Aicher/Germany won silver, and Sofia Goggia/Italy won bronze. 
This directly revealed the answer before the simulated agent should have had access to the result, motivating our use of a locally frozen, date-gated CCNews corpus instead of relying on external search APIs.

\subsection{\Cref{tab:benchmark-comparison} Columns}
\label{app:tableexplanation}
We mark benchmarks as \textit{open-ended}, in which the agent must decide what to do rather than being given a single, fully defined task. Benchmarks that can benefit from sequentially learning across tasks are marked as satisfying \textit{testing adaptation}. We estimate \textit{horizon lengths} as the maximum actions used in a single model trajectory reported in the original papers, with N/R = not reported.

\crefalias{section}{appendix}
\section{Metrics}
\label{app:metrics}

Following the notation from \Cref{sec:futuresim}, for a question $q\in\mathcal{Q}$, the ground-truth answer is $y_q$, and an agent forecast is a set of outcomes $\Omega_q=\{o_1,\ldots,o_k\}$ with probabilities $p_q(o_i)$ satisfying $p_q(o_i)\geq 0$ and $\sum_i p_q(o_i)\leq 1$. We assume all the outcomes in $\Omega_q$ are (semantically) distinct from each other. We set $p_q(o)=0$ for outcomes not named in $\Omega_q$. In our implementation, an answer matcher checks an outcome's semantic equality to $y_q$. As described in~\Cref{sec:futuresim}, the score for a resolved question is defined as:
\[
\mathrm{BSS}(q)
=
1 -
\sum_{o \in \Omega_q \cup \{y_q\}}
\left(p_q(o) - \mathbf{1}[o = y_q]\right)^2
\]
Higher is better: a fully confident correct forecast receives $1$, an abstention with no probability mass receives $0$, and assigning all the probability to wrong outcomes receives $-1$. The CSV maintaining the tasks also contains a daily-updating \texttt{avg\_brier} column, which shows the mean of $\mathrm{BSS}(q)$ over all questions, with questions that have no forecast contributing $0$. Once a question resolves, its final score is held fixed in later daily aggregates.

\subsection{Worked Example}
\label{app:bss-example}

To illustrate how the Brier skill score behaves in practice, consider the question \textit{``Which NFL team will win Super Bowl LX in February 2026?''} with ground-truth answer $y_q = \text{Seattle Seahawks}$. We compute the score for two hypothetical forecasts.

\paragraph{Forecast A (3 outcomes, positive BSS).}
The agent reports the outcome set
\[
\Omega_q^A = \{\text{Seattle Seahawks}: 0.55,\ \text{Chiefs}: 0.25,\ \text{49ers}: 0.10\},
\]
which sums to $0.90$. Since $y_q = \text{Seattle Seahawks}\in\Omega_q^A$, we have $\Omega_q^A\cup\{y_q\}=\Omega_q^A$, and
\[
\begin{aligned}
\mathrm{BSS}(q)
&= 1 - (0.55 - 1)^2 - (0.25 - 0)^2 - (0.10 - 0)^2 \\
&= 1 - 0.2025 - 0.0625 - 0.01 \\
&= 0.725.
\end{aligned}
\]
The agent confidently named the correct answer at the top of its prediction set and assigned only modest mass to plausible alternatives, earning a high positive score.

\paragraph{Forecast B (4 outcomes, negative BSS).}
Now consider a forecast that still names the correct outcome but commits most of its probability to a confident wrong guess:
\[
\Omega_q^B = \{\text{Chiefs}: 0.55,\ \text{Bills}: 0.20,\ \text{Seattle Seahawks}: 0.15,\ \text{49ers}: 0.10\},
\]
which sums to $1.00$. The truth (Seattle Seahawks) is included in $\Omega_q^B$, so an evaluator that only checked whether the right answer appeared anywhere in the prediction might consider this forecast to ``have the answer''. However, the agent's top-1 outcome is Chiefs, not the Seattle Seahawks. The score is
\[
\begin{aligned}
\mathrm{BSS}(q)
&= 1 - (0.55 - 0)^2 - (0.20 - 0)^2 - (0.15 - 1)^2 - (0.10 - 0)^2 \\
&= 1 - 0.3025 - 0.04 - 0.7225 - 0.01 \\
&= -0.075.
\end{aligned}
\]
Despite listing the correct team, the Brier skill score is negative: the dominant penalty $(0.15-1)^2 = 0.7225$ comes from leaving the Seattle Seahawks with only $0.15$ probability while committing $0.55$ to Chiefs, and the additional hedges on Bills and the 49ers each contribute small but strictly positive squared-error terms. %
Thus, the Brier skill score rewards correct \emph{and} well-allocated probabilities, and penalizes confident commitment to a wrong outcome more strongly than abstention (which would lead to a score of $0$). %

\subsection{Properness of Brier Skill Score}
\label{app:multibrier}

\begin{theorem}[Properness of Brier skill score for subdistributions]
\label{lem:bss-subdistribution-proper}
Fix a question $q$, and suppress the subscript $q$. Let $\Omega$ be the reported
outcome set, and let $\Omega'$ be the outcome set on which the forecaster's true
belief is defined. Suppose the answer matcher induces a partial one-to-one
matching between $\Omega$ and $\Omega'$: each $o\in\Omega$ matches at most one
$y\in\Omega'$, and each $y\in\Omega'$ matches at most one $o\in\Omega$.

Let the report be a subdistribution
\[
p:\Omega\to[0,1],
\qquad
\sum_{o\in\Omega}p(o)\leq 1,
\]
and let the true belief be a subdistribution
\[
\pi':\Omega'\to[0,1],
\qquad
\sum_{y\in\Omega'}\pi'(y)\leq 1.
\]
Any missing probability mass is treated as mass on outcomes that match no
reported outcome.

Define the projection $\bar{\pi}$ of the true belief onto $\Omega$ by
\[
\bar{\pi}(o)
=
\begin{cases}
\pi'(y), & \text{if there is a unique } y\in\Omega' \text{ such that } y\sim o,\\
0, & \text{otherwise.}
\end{cases}
\]
Then, before knowing which outcome will occur, the Brier skill score is strictly proper with respect to $\bar{\pi}$:
\[
\mathbb{E}[\mathrm{BSS}(p,Y)]
\leq
\mathbb{E}[\mathrm{BSS}(\bar{\pi},Y)],
\]
with equality if and only if
\[
p=\bar{\pi}.
\]
\end{theorem}

\begin{proof}
Since $\pi'$ is a subdistribution and $\bar{\pi}$ is its projection onto
$\Omega$,
\[
\sum_{o\in\Omega}\bar{\pi}(o)\leq 1.
\]
Let
\[
r
=
1-\sum_{o\in\Omega}\bar{\pi}(o).
\]
Thus $r$ is the total probability that the realized answer matches no reported
outcome.

For any report $p$, write
\[
\|p\|_2^2
=
\sum_{o\in\Omega}p(o)^2.
\]
If the realized answer matches some $o\in\Omega$, then
\[
\mathrm{BSS}(p,Y)
=
1-\left((p(o)-1)^2+\sum_{a\in\Omega:\,a\neq o}p(a)^2\right)
=
2p(o)-\|p\|_2^2.
\]
If the realized answer matches no reported outcome, then
\[
\mathrm{BSS}(p,Y)
=
1-\left(1+\sum_{a\in\Omega}p(a)^2\right)
=
-\|p\|_2^2.
\]
Therefore,
\[
\begin{aligned}
\mathbb{E}[\mathrm{BSS}(p,Y)]
&=
\sum_{o\in\Omega}
\bar{\pi}(o)\left(2p(o)-\|p\|_2^2\right)
+
r\left(-\|p\|_2^2\right) \\
&=
2\sum_{o\in\Omega}\bar{\pi}(o)p(o)
-
\left(\sum_{o\in\Omega}\bar{\pi}(o)+r\right)\|p\|_2^2 \\
&=
2\sum_{o\in\Omega}\bar{\pi}(o)p(o)-\|p\|_2^2.
\end{aligned}
\]
Using the identity
\[
\|p-\bar{\pi}\|_2^2
=
\|p\|_2^2
-
2\sum_{o\in\Omega}\bar{\pi}(o)p(o)
+
\|\bar{\pi}\|_2^2,
\]
we get
\[
\mathbb{E}[\mathrm{BSS}(p,Y)]
=
\|\bar{\pi}\|_2^2-\|p-\bar{\pi}\|_2^2.
\]
The term $\|\bar{\pi}\|_2^2$ does not depend on $p$, while
\[
\|p-\bar{\pi}\|_2^2\geq 0,
\]
with equality if and only if $p=\bar{\pi}$. Hence the expected score is uniquely
maximized by reporting $p=\bar{\pi}$.
\end{proof}

\subsection{Time-Weighted and Peer Scores}
\label{app:timeweightedmetrics}

Let $p_{q,t}$ denote the forecast for question $q$ that an agent is holding on day $t$, using the same notation as above at each snapshot. Let $T_q=\{o_q,\ldots,r_q-1\}$ be the days on which question $q$ is open, and let $\mathrm{BSS}_t(q)$ be the Brier skill score of $p_{q,t}$, or $0$ if no forecast has yet been submitted. The single-agent time-weighted score is defined as:
\[
\mathrm{TW}
=
100
\sum_{q\in\mathcal{Q}}
\frac{1}{|T_q|}
\sum_{t\in T_q}
\mathrm{BSS}_t(q)
\]
This rewards positive brier forecasts that are made earlier (and also penalizes negative forecasts made earlier), as a forecast earns its score on every day it is held until it is updated or the question resolves.

In multi-agent runs, each daily Brier skill score is first made peer-relative. For agent $a$, we define:
\[
\mathrm{Peer}_{a,t}(q)
=
\mathrm{BSS}_{a,t}(q)
-
\overline{\mathrm{BSS}}_{-a,t}(q)
\]
where $\overline{\mathrm{BSS}}_{-a,t}(q)$ is the average score of the other agents with active forecasts on $q$ at day $t$; if there is no other active forecast, the baseline is $0$. The time-weighted peer score applies the same single $100$ multiplier as the single-agent time-weighted score:
\[
\mathrm{TWPeer}_a
=
100
\sum_{q\in\mathcal{Q}}
\frac{1}{|T_q|}
\sum_{t\in T_q}
\mathrm{Peer}_{a,t}(q)
\]

\subsection{Aggregate Predictions and Update Size}
\label{app:multiagentaggtv}

When multiple agents forecast the same question, the market aggregate exposed in the environment is the coordinate-wise mean of their current probabilities:
\[
\bar p_q(o)
=
\frac{1}{n_q}
\sum_{a=1}^{n_q} p_q^{(a)}(o)
\]
where missing outcomes are assigned a probability of $0$.

For two forecasts $p_q$ and $p'_q$ with outcome sets $\Omega_q$ and $\Omega'_q$, the total variation distance is given by:
\[
d_{\mathrm{TV}}(p_q,p'_q)
=
\frac{1}{2}
\sum_{o\in\Omega_q\cup\Omega'_q}
\left|p_q(o)-p'_q(o)\right|
\]
We report this metric when studying multi-agent dynamics in \Cref{app:multiagentperformance}.

\subsection{Variance Computation}
\label{app:variancecomputation}

For time-series plots, the error bands are computed from question-level scores rather than from averaged daily scores. For each model, we take the intersection of the questions answered across runs. We estimate uncertainty by bootstrapping over both questions and runs. For each time step, we resample questions with replacement from the common question set. For each sampled question, we randomly select one run and use that run's score for the question on that date. We average these scores to obtain a single bootstrapped estimate. The shaded band in \Cref{fig:figure1} shows one standard deviation of these estimates around the mean (thick line).

\crefalias{section}{appendix}
\section{Additional experiments}
\label{app:additional-experiments}

\subsection{Environment Interaction}

\Cref{app:toolcalls} shows the number of actions taken in the environment across models. Note that we do not restrict the number of actions, allowing the models to run indefinitely with as many context window compactions as they want. In fact, we had to prompt the open-weight models separately to keep taking actions in our harness to achieve better results, because in initial runs they would just keep calling \texttt{next\_day}. GPT 5.5 achieves the best performance, but also take the largest number of actions in the environment. Similarly, Qwen 3.6 Plus takes the least actions and has worst performance. Yet, the number of actions doesn't fully explain performance, as while Opus 4.6 and DeepSeek V4 Pro use a similar number of actions, Opus 4.6 performs substantially better. Similarly with GLM 5.1 and Qwen3.6 Plus, where the former performs much better. It is however noticeable that the agents that benefited most from test-time adaptation in our ablation experiment (\Cref{fig:adaptation-memory}) consistently make more actions during the simulation in the main benchmark runs.

\begin{figure}[htbp]
    \centering
    \includegraphics[width=0.8\linewidth]{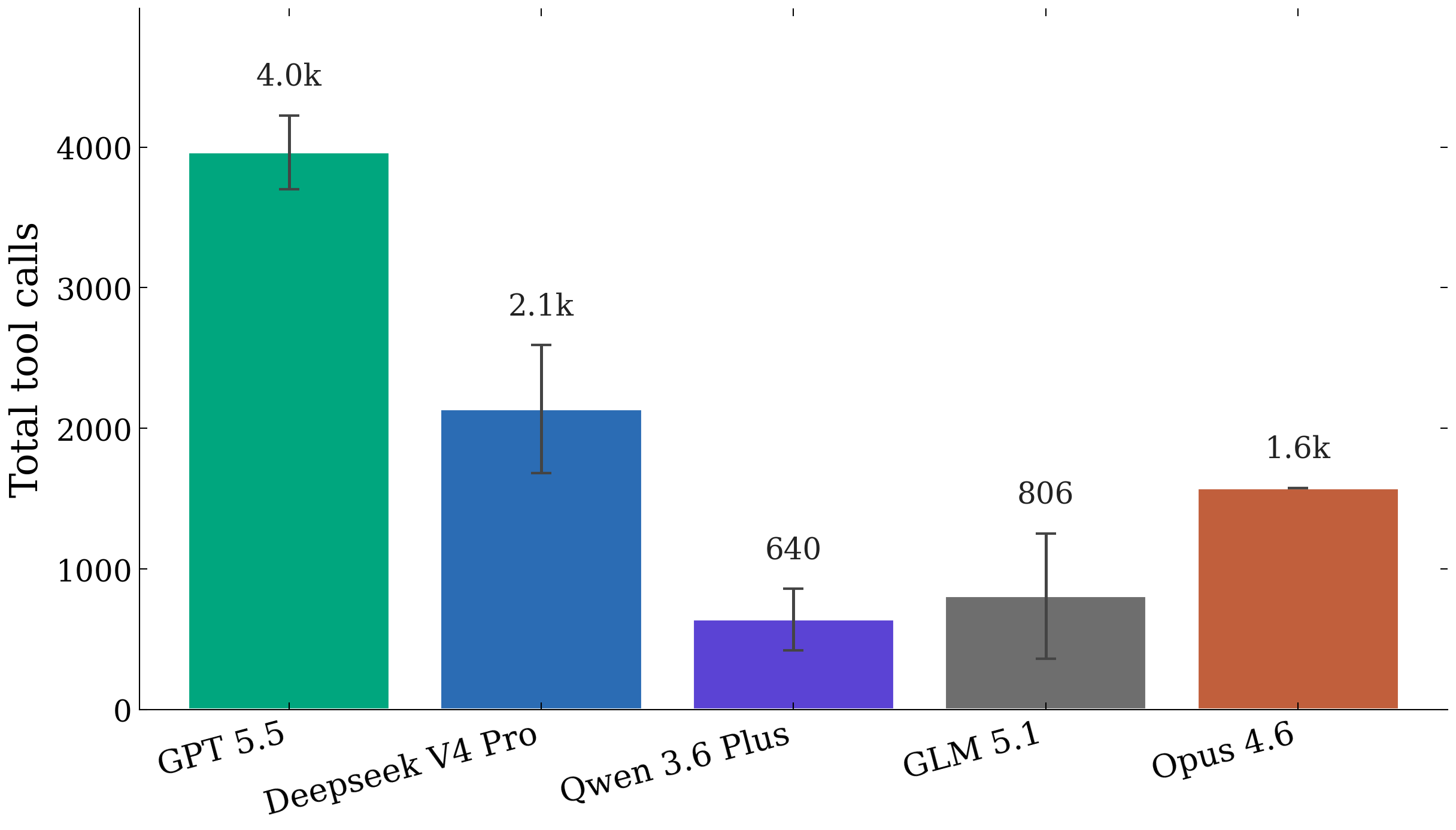}
    \caption{\textbf{Number of actions.} We report the total number of actions for each model during the simulation. The results show how long-horizon our benchmark is, with GPT 5.5 in Codex, the best performing agent taking around 4,000 actions across runs ranging over multiple context window compaction calls. We find that the number of actions taken by different agents is correlated with the test-time adaptation improvement reported in \Cref{fig:adaptation-memory}, but not always the absolute brier skill score.}
    \label{app:toolcalls}
\end{figure}

\subsection{Performance on Questions where all models update predictions}
\label{app:subset_questions}

A feature of our benchmark is that models can abstain from making any predictions on some questions, and in this case for that question they are assigned accuracy 0 and Brier skill score 0 (where the latter is better than making a \textit{misleading} prediction). This also means that question selection can play a role in the final performance. We adopt this choice to keep the task set scalable in the future, as in the real world humans also have to choose what tasks they focus on, though it is small enough in our experiments (330 questions) that all models except Qwen 3.6 Plus do make predictions on all questions. 

This can however still skew our test-time adaptation experiments, where there is a significant difference in the number and questions different models update on. Thus, we now report a variant of the experiment in \Cref{fig:adaptation-memory}, where we measure test-time adaptation on the subset of questions where every model makes atleast one prediction \textit{update} after the initial prediction on the question. This leaves only $46$ questions, with results reported in \Cref{app:subset_questions_fig}.

\begin{figure}[t]
    \centering
    \begin{minipage}[t]{0.49\linewidth}
        \centering
        \includegraphics[width=\linewidth]{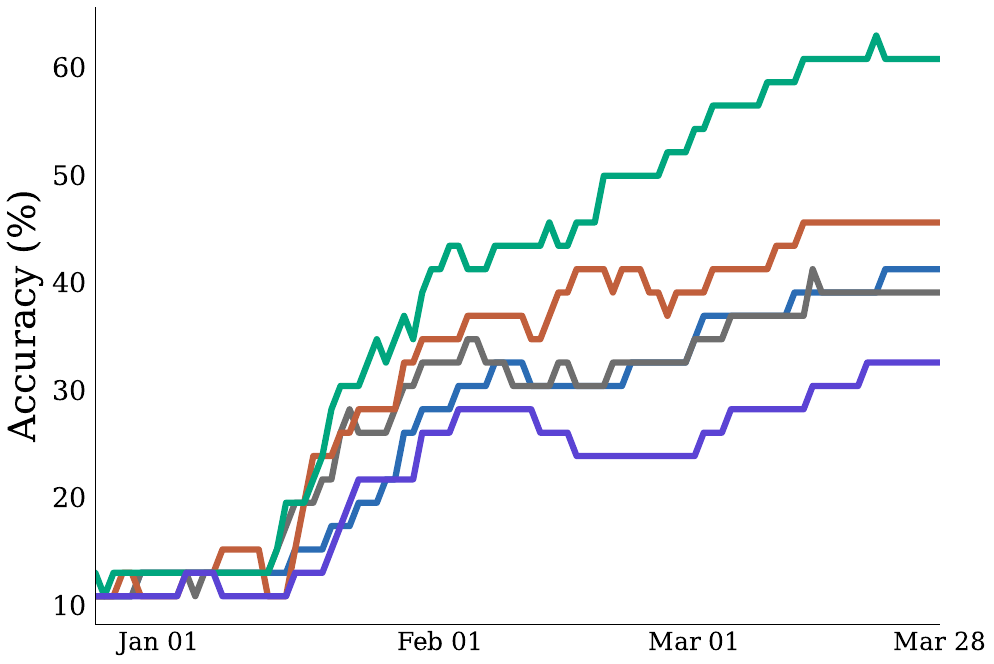}
    \end{minipage}
    \hfill
    \begin{minipage}[t]{0.49\linewidth}
        \centering
        \includegraphics[width=\linewidth]{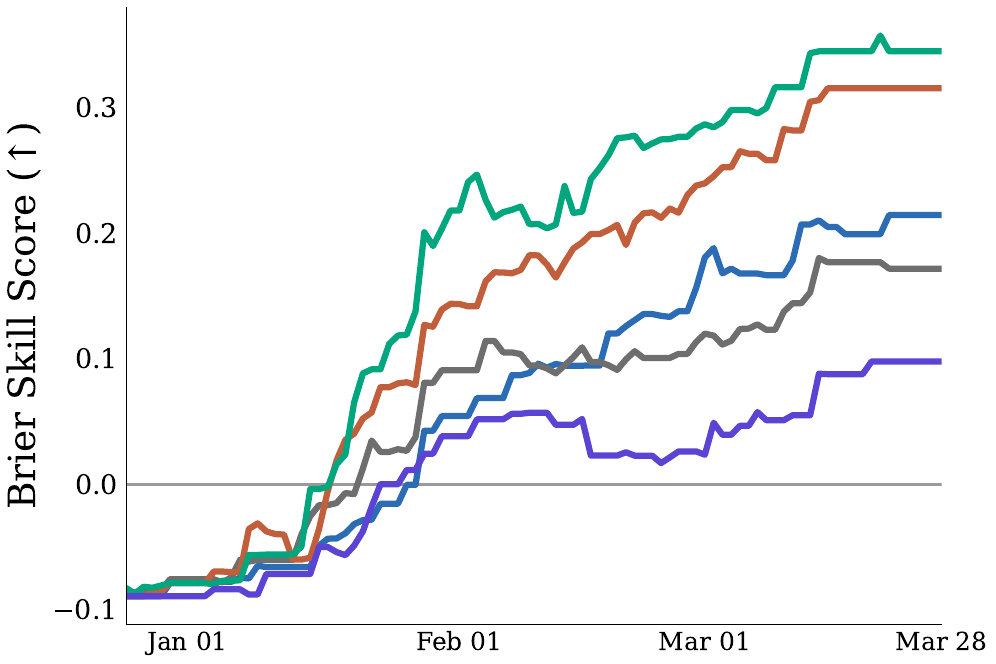}
    \end{minipage}
    \caption{\textbf{Test-time adapation on the subset where all models make prediction updates.} We once again start with Qwen 3.6 Plus predictions on all questions, and then restrict the analysis to the subset of $46$ questions where all models submit at least one forecast. Across both accuracy and brier skill score, we see consistent trends in test-time adaptation with the main paper plot shown over the full question pool, showing question selection currently does not lead to a significant change in model rankings.}
    \label{app:subset_questions_fig}
\end{figure}

We see that for both accuracy and brier skill score, rankings remain consistent with those computed over the full set of 330 questions as reported in \Cref{fig:adaptation-memory}. We do notice that the absolute accuracy and brier skill score values grow very fast on this subset where all models made updates. This matches our expectation, as if all models made an update, there likely was strong new evidence toward the ground-truth to prompt these updates.

\subsection{Search Ablations Brier skill score}

\Cref{app:searchablationbrier} shows the Brier skill score version of the search
ablation in \Cref{fig:search}. It preserves the main
qualitative conclusion that fresh daily context and agentic search matter. The
full simulation with daily additions to the news corpus is much better than freezing the corpus at
day 0. In the case of the latter, agent updates just keep reinforcing confidence in predictions without new evidence, leading to significant overconfidence and an eventual negative Brier skill score. Interestingly, making a single retrieval query also leads to highly negative brier skill score, showing the importance of iterative refinement via search for calibrated beliefs~\citep{murphy2026agenticforecastingusingsequential}. Unlike the accuracy plot in the main paper (\Cref{fig:search}), testing each question independently one day before its resolution does not lead to significant improvements in Brier skill score than sequential updates over the course of the simulation.

\begin{figure}[htbp]
    \centering
    \includegraphics[width=0.4\linewidth]{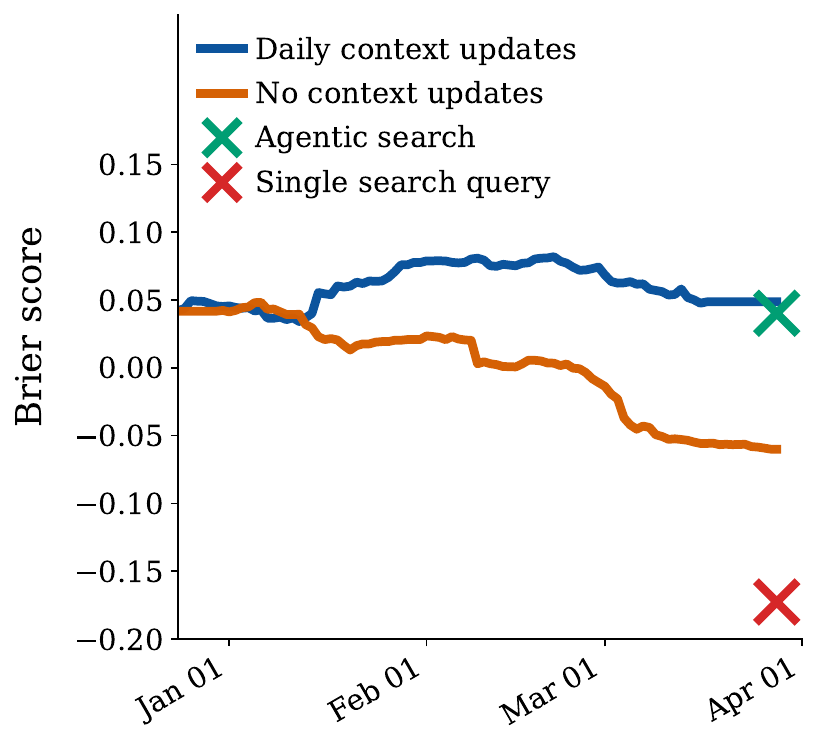}
    \caption{\textbf{Benefits from search.} We evaluate GPT 5.5 xhigh reasoning effort in four different settings to isolate the benefits of agentic search over updating context in \futuresim, this time measuring brier skill score. Consistent with the accuracy trend, we find large improvements from daily context updates (blue line) compared to when no articles beyond the first date are added during the simulation (orange line), and agentic search (green cross) outperforms just performing a single search query (red cross) one day before each question resolves. However, evaluating each question independently one day before its resolution does not improve brier skill score despite the large gain in accuracy, which can be attributed to poor distribution of probabilities over possible outcomes.}
    \label{app:searchablationbrier}
\end{figure}

\subsection{Inference Scaling Brier skill score}

\Cref{app:inferencescalingbrier} mirrors the inference-scaling accuracy result
in \Cref{fig:inference-scaling}. Increasing GPT 5.5's
reasoning effort improves Brier skill score as well as accuracy, with the
largest gains coming from moving away from no reasoning and diminishing returns
at the highest efforts. This supports the same interpretation as the main
figure: extra test-time compute helps the agent search and reason, but the
benefit plateaus for this model in \futuresim. Notably, we see highly negative Brier skill score at reasoning effort ``none''. This indicates that more tokens at test-time help the model make more calibrated predictions, which is corroborated by qualitative analysis of reasoning summaries for closed-weight, and transparent reasoning chains for open-weight models, where models spend a significant amount of tokens deciding what probability to output.

\begin{figure}[htbp]
    \centering
    \includegraphics[width=0.4\linewidth]{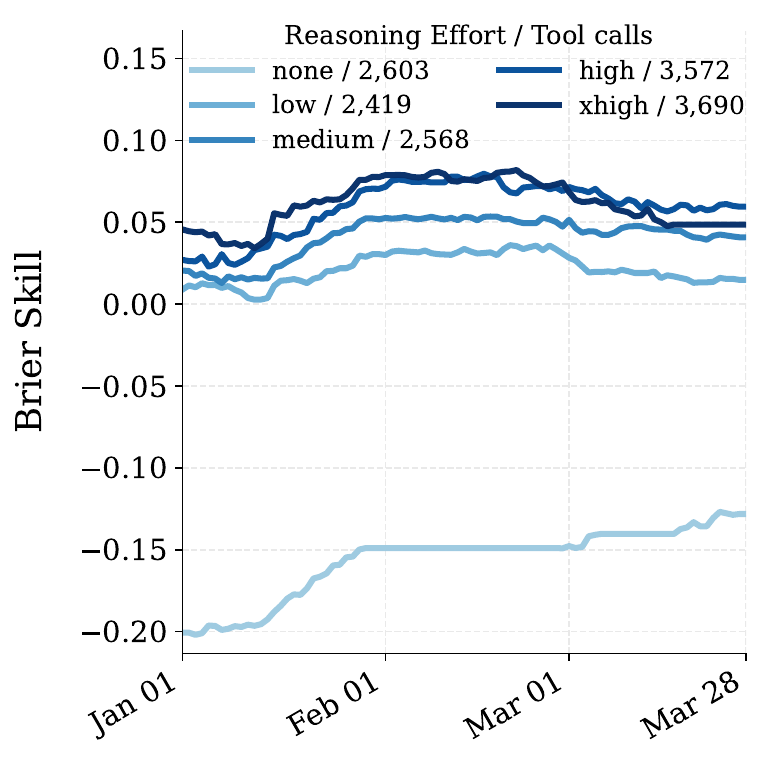}
    \caption{\textbf{Effect of scaling test-time compute on brier skill score.} We run GPT 5.5 in all five available reasoning efforts to see how additional inference compute changes brier skill score on \futuresim. We find higher reasoning effort consistently leads to better brier skill score, although the effect plateaus for this model after reasoning effort high. Notably, reasoning effort ``none'' has extremely poor brier skill score while not being as bad in accuracy, indicating the importance of reasoning for calibration~\citep{damani2026beyond}.}
    \label{app:inferencescalingbrier}
\end{figure}

\subsection{Multi-agent performance}

\label{app:multiagent-performance}

\Cref{app:multiagentperformance} reports the accuracy and Brier skill score
trajectories for the multi-agent experiment described in \Cref{sec:multiagent}.
The multi-agent runs seem to start at and maintain higher accuracy than independent single agent runs
for all three DeepSeek v3.2 agents. While deeper exploration of this phenomenon is left to future work due to cost reasons, one hypothesis for why the initial predictions (before any inter-agent interaction through the crowd aggregate) could be better is that agents were told they are competing with other agents in the environment in the multi-agent setting. 

\begin{figure}[htbp]
    \centering

    \begin{minipage}{0.49\linewidth}
        \centering
        \includegraphics[width=\linewidth]{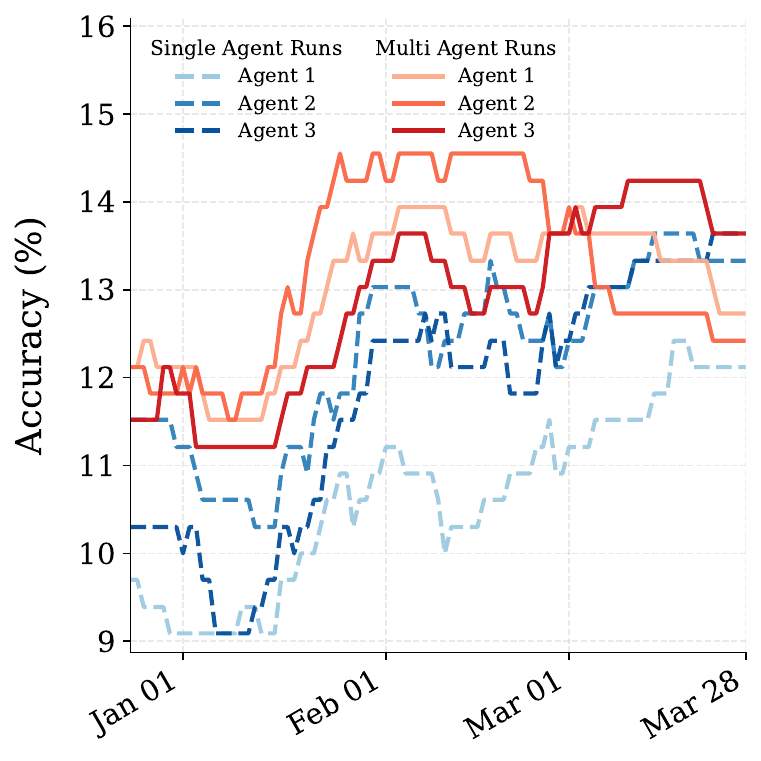}
    \end{minipage}
    \hfill
    \begin{minipage}{0.49\linewidth}
        \centering
        \includegraphics[width=\linewidth]{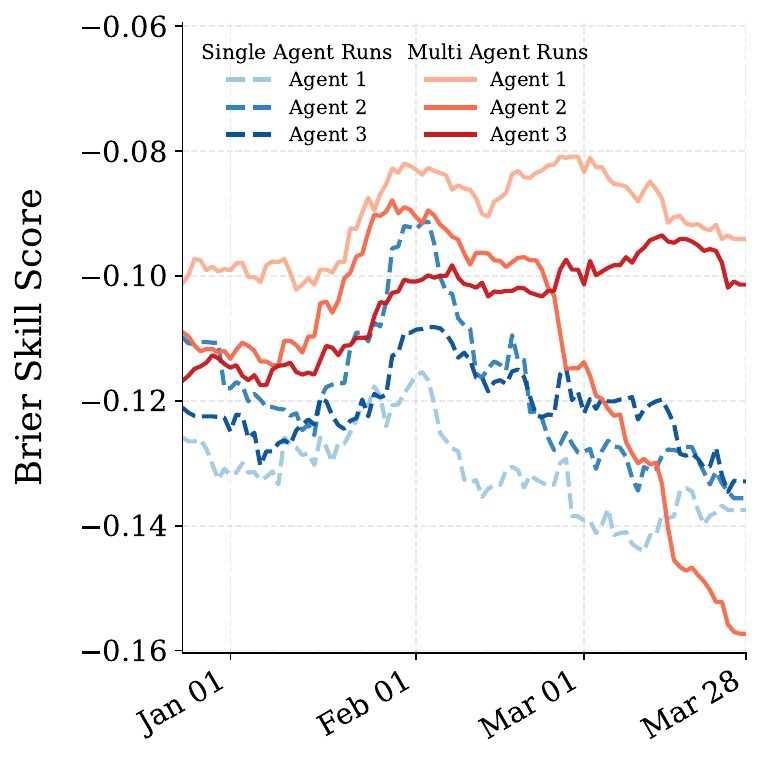}
    \end{minipage}

    \caption{\textbf{Multi-agent performance.} When we run multiple copies of DeepSeek v3.2 agents simultaneously, we see anecodtal evidence that individual agent predictions obtain slightly higher accuracy than the independent agent runs.}
    \label{app:multiagentperformance}
\end{figure}

\crefalias{section}{appendix}

\section{Prompts}
\label{app:prompts}

\subsection{Native Harness}
\label{app:prompt-native-harness}

The following template is rendered for the main forecasting runs. Runtime fields such as dates and question counts are filled in by the simulation harness.

\begin{promptframe}{}
\begin{PromptVerbatim}
You are a forecasting agent. Today is <current_date>. Your goal is to make accurate and calibrated predictions.

## UPDATE CADENCE
You have the chance to update your predictions every <timegap_days> day(s). Your workspace files (memory/, scripts, notes) persist across days -- use them to track reasoning and lessons learned. Articles are available via the search tool and in the articles/ directory. Current date: <current_date>. Next scheduled update: <next_date>.
<optional_resolution_cadence_note>

## SCORING (Brier Skill Score)
You have to output a distribution of (outcome, probability) pairs for each question you make a forecast on.
You are evaluated on the Brier Skill Score = 1 - sum_i (p_i - y_i)^2 summed over all outcomes, where:
- p_i = your probability for outcome i
- y_i = 1 if your outcome i is TRUE, 0 otherwise
- Higher is better: 1.0 = perfect, 0.0 = abstaining from guessing, negative = worse than abstaining.

Key Mechanics:
1. Accuracy + Calibration: assign calibrated probabilities that reflect true likelihoods.
2. Time-Weighted Score: forecasts made earlier matter, but updating is rewarded when new evidence arrives.
3. Prediction-Count Incentive: unanswered active questions receive zero contribution.
4. End-of-Session Metrics are shown after each session.
5. Max Outcomes: submit at most <max_outcomes_per_question> outcomes per question.
6. No Placeholders: "Unknown", "TBD", and "Other" hurt your score.

## AVAILABLE DATA
You have access to a news article database, which is updated daily through a search tool, that you can use to find evidence for your forecasts.
You can access the market.csv file (READ-ONLY) in your workspace containing <num_questions> questions (<num_active> active/unresolved, <num_resolved> resolved).

Column descriptions of the DataFrame (market.csv):
- qid, title, background, resolution_criteria, answer_type
- resolution_date, is_resolved, ground_truth
- num_predictions, options, my_prediction, my_prediction_date

Note: my_prediction contains your current forecast as a dict, or None if not yet predicted. ground_truth contains the resolved answer, or None if not yet resolved.

## TOOLS AVAILABLE FOR YOUR USE
- mcp__forecast__search_news(query, from_date?, to_date?): search the news corpus for evidence. to_date is capped at today's date.
- mcp__forecast__submit_forecasts(question_id, outcomes): submit exactly one forecast for exactly one question ID.
- mcp__forecast__next_day(): end the current session and proceed to the next one.

## Workspace:
- market.csv -- Read-only snapshot of all questions, refreshed each day.
- articles/ -- Browsable news articles organized by date as articles/YYYY/MM/DD/articles.jsonl.
- memory/ -- Persistent notes directory. Read and write freely.

You have full control over your workspace. You may create files or directories that help you perform better, such as forecasting strategies, calibration notes, scripts, or per-question research notes.

## SUBMISSION RULES
- qid must be from an active unresolved question identified from market.csv.
- Each mcp__forecast__submit_forecasts call must contain exactly one forecast for one question ID.
- You may submit again later in the same session to update that qid.
- Maximum of <max_outcomes_per_question> outcomes allowed per question.
- Outcome names must be real predicted answers.
- Never use placeholders like "Unknown", "TBD", "Other", or "N/A".
- Probabilities must sum to <= 1.0.

## Rules
- No web access is available. Use mcp__forecast__search_news and articles/ for information.
- market.csv is read-only. Do not modify it.
- You can use Bash, Read, Write, Grep, Glob, and other tools freely in your workspace.
- Your job is to maximize your time-weighted score.

Begin.
\end{PromptVerbatim}
\end{promptframe}

\subsection{Our Harness}

This prompt is used for our custom forecasting harness. It adds daily feedback, memory tools, update priorities, and action/context-budget guidance around the same forecasting task.

\begin{promptframe}{}
\begin{PromptVerbatim}
You are a forecasting agent. Today is <current_date>. Your goal is to make accurate and calibrated predictions.

<results_since_last_session>
## RESULTS SINCE YOUR LAST SESSION (<last_date> -> <current_date>)
- "<resolved_question_title>"
  Your prediction distribution: <prediction_distribution> | Truth: <ground_truth>
  Brier: <brier_skill_score> | TW-Score: <time_weighted_score>
...

## YOUR CUMULATIVE PERFORMANCE TILL TODAY
- Total Predictions: <num_predictions> (<num_resolved> resolved)
- accuracy: <accuracy>%
 accuracy = fraction of resolved questions where your top outcome matched the truth; brier skill score = mean brier skill score across resolved questions

<source_context>
<source_specific_rules>

## UPDATE CADENCE
You can make updates every <timegap_days> day(s). Your context is cleared after every session and your memory (along with past predictions) is the only information retained between sessions. <new_articles_text><last_update_text>Current date: <current_date>. Next scheduled update: <next_date>.
<tomorrow_resolution_reminder>

IMPORTANT: You have predictions on <predicted_count> out of <active_count> active questions.
Tip: You can check your existing predictions by reading market.csv and filtering rows where `my_prediction` is not null.

UPDATE RULES:
- Do NOT re-predict questions from scratch unless you find specific new evidence.
- Only update a prediction if you find SPECIFIC NEW evidence (news, data) that updates your view.

PRIORITIES FOR UPDATES:
1. Questions resolving the next day (filter `market.csv` by `resolution_date` == tomorrow) -- make sure your prediction is up-to-date before calling next_day.
2. Questions without predictions (if any)
3. Questions where today's news search reveals new information
4. Questions approaching resolution date that you haven't checked recently
5. Skip questions where there is no new evidence

## YOUR MEMORY
Current meta-insights with their indices:
<meta_insight_index>

`mem_df` holds your per-question notes (reasoning, evidence, calibration) -- 1 row per question.
Columns: qid (str), question (str), last_updated (str), memory (str), category (str)
<prior_memory_location_or_empty_memory_note>

Inspect `mem_df` by reading <prior_mem_csv>. Edit per-question notes with `mcp__forecast__mem_add`, `mcp__forecast__mem_update`, `mcp__forecast__mem_delete`.
Manage meta-insights with `mcp__forecast__memory_retrieve` (using the indices), `mcp__forecast__memory_new`, `mcp__forecast__memory_update`, `mcp__forecast__memory_delete`.
Caps: meta-insights <= 500 entries; per-question `mem_df` memory <= 1000 chars per row.

## SCORING (Brier Skill Score)
You have to output a distribution of (outcome, probability) pairs for each question you make a forecast on.
You are evaluated on the Brier Skill Score = 1 - sum_i (p_i - y_i)^2 summed over all outcomes, where:
- p_i = your probability for outcome i
- y_i = 1 if your outcome i is TRUE, 0 otherwise
- Higher is better: 1.0 = perfect, 0.0 = abstaining from guessing, negative = worse than abstaining.

Key Mechanics:
1. Accuracy + Calibration: assign calibrated probabilities that reflect true likelihoods.
2. Time-Weighted Score: forecasts made earlier matter, but updating is rewarded when new evidence arrives.
3. Prediction-Count Incentive: unanswered active questions receive zero contribution.
4. End-of-Session Metrics are shown after each session.
5. Max Outcomes: submit at most <max_outcomes_per_question> outcomes per question.
6. No Placeholders: "Unknown", "TBD", and "Other" hurt your score.

## AVAILABLE DATA
You have access to a news article database which is updated daily through a search tool, that you can use to find evidence for your forecasts.
You also have access to a read-only `market.csv` file in your workspace with <num_questions> questions (<num_active> active/unresolved, <num_resolved> resolved).

Column descriptions of market.csv:
- qid, title, background, resolution_criteria, answer_type
- resolution_date, is_resolved, ground_truth
- num_predictions, options, my_prediction, my_prediction_date

Note: `my_prediction` contains your current forecast as a dict, or None if not yet predicted. `ground_truth` contains the resolved answer, or None if not yet resolved.

## TOOLS AVAILABLE FOR YOUR USE
- `mcp__forecast__search_news(query, from_date?, to_date?)`: search the news corpus for evidence. `to_date` is capped at today's date.
- `mcp__forecast__memory_retrieve` / `mcp__forecast__memory_new` / `mcp__forecast__memory_update` / `mcp__forecast__memory_delete`: manage meta-insight entries.
- `mcp__forecast__mem_add` / `mcp__forecast__mem_update` / `mcp__forecast__mem_delete`: manage question-specific notes in `mem_df`.
- `mcp__forecast__submit_forecasts(question_id, outcomes)`: submit exactly one forecast for exactly one question ID (`qid`).
- `mcp__forecast__next_day()`: first call enters memory-update mode; call it a second time after your memory updates to actually proceed to the next day.

You also have access to native tools Bash/Read/Grep etc. -- use them to read market.csv and browse articles/. The MCP server persists today's `mem.csv` + `meta.yaml` automatically on `mcp__forecast__next_day`; do not write under `memory/` yourself.

## Workspace:
- market.csv -- Read-only snapshot of all questions, refreshed each day.
- articles/ -- Browsable news articles organized by date as articles/YYYY/MM/DD/articles.jsonl.
- predictions/ -- Read-only record of your past submissions, one file per day as `predictions/YYYY-MM-DD.json`.
- memory/ -- Read-only persisted notes (`memory/YYYY-MM-DD/{mem.csv, meta.yaml}`). Read prior days' files for context; edit memory only through the MCP memory tools.

## INTERACTION FLOW
You have <max_actions> actions per day. Each query, search, memory operation, or submission uses 1 action.
You have a context budget of <max_total_tokens> tokens for this session. This tracks both the input prompt and cumulative output tokens spent so far.
Keep at least <submit_reserve_tokens> tokens free for a final submit. Force-submit once the remaining context budget is at or below <force_submit_threshold_tokens>.
If both budgets are configured, both are enforced and the session ends when either one is exhausted.

You can interleave reads, searches, memory operations, and submissions as needed. Consider reading `memory/<last_date>/mem.csv` early to recall prior reasoning and identify which questions need attention.

## MEMORY WORKFLOW
Treat memory as two layers:
- `mem_df`: question-specific reasoning, evidence, and calibration notes for a single QID.
- meta-insights: reusable cross-question patterns, lessons, and calibration rules that should help on future days.

Before calling `mcp__forecast__next_day()`:
1. Update `mem_df` for questions you researched or forecasted today using `mcp__forecast__mem_add` / `mcp__forecast__mem_update`.
2. If today's work revealed a reusable pattern, lesson, or calibration rule, promote it into a meta-insight.
3. If a prior meta-insight is stale or contradicted, revise or delete it.

Do not use meta-insights as a daily activity log. If you learned nothing reusable today, it is fine to skip meta-insight writes.

## SUBMISSION RULES
- qid must be from an active (`is_resolved=False`) question you identified from market.csv.
- Each `mcp__forecast__submit_forecasts` call must contain exactly one forecast for one question ID.
- You may submit again later in the same session to update that qid.
- Maximum of <max_outcomes_per_question> outcomes allowed per question.
- Outcome names must be real predicted answers.
- Never use placeholders like "Unknown", "TBD", "Other", or "N/A".
- Probabilities must sum to <= 1.0.

Tip: After submitting a forecast, consider saving your reasoning and key evidence for that QID using `mcp__forecast__mem_add`/`mcp__forecast__mem_update`.

---
Budget at start:
Actions remaining: <max_actions>
Context tokens remaining: <max_total_tokens> (estimated current context 0 / <max_total_tokens>)
Note: current message tokens are not accounted for yet.
Begin.
\end{PromptVerbatim}
\end{promptframe}

\subsection{Multi-Agent}
\label{app:prompt-multi-agent}

When the simulation is run with multiple agents, the daily prompt is modified as follows. These changes are inserted in addition to the base forecasting prompt.

\begin{promptframe}{}
\begin{PromptVerbatim}
## MULTI-AGENT SETTING
You are competing against <N-1> other forecasting agents on the same set of questions.
You each predict independently on every wakeup day. After each day, your predictions
are averaged with the others' into a market aggregate (the `market_aggregate`
column), which you can see starting the following day.
You are scored relative to your competitors: to earn a positive time-weighted peer
score, your predictions need to be more accurate than the group average.

...

## SCORING (Time-Weighted Peer Score (Brier-Skill Based))

...

- **Time-Weighted Peer Score (TW-Peer)**: On each day a prediction is held, your
Brier Skill Score is compared to the mean of all other agents' scores for the same
question. These daily differences are summed over the lifetime of the prediction.
A positive TW-Peer indicates predictions that were consistently more accurate than
the group average.

...

**Relative Performance (multi-agent)**: Final scoring is relative, so you have to
outperform the market aggregate to gain positive peer score.

...

Note: `market_aggregate` and `my_prediction` columns contain Python dicts
(or None). You can access them directly, e.g.
`row['market_aggregate']['outcome_name']`.

- `market_aggregate`: the mean probability distribution across all agents' latest
predictions from the previous day. `None` on the first day.
- `my_prediction`: your own latest forecast, or None if you have not predicted this
question yet.
- `num_predictions`: total number of prediction submissions made on this question
across all agents and all days.
\end{PromptVerbatim}
\end{promptframe}

\subsection{Earliest-Date Repair Prompt}
\label{app:prompt-late-resolve}

The following prompt is used to figure out potentially already-resolved questions and infer the earliest date on which the answer could have been determined from public information.

\begin{promptframe}{}
\begin{PromptVerbatim}
You are provided with a forecasting question (which might be from the past). You have to find not only the answer to the question, but also the earliest date on which the answer to the question could be inferred. Be smart in your inference. The question might contain extra details about the situation/event being asked but I want you to find out the earliest date by which the answer could have been figured out (even without extra details). For example, if you had seen the question 6 months back, could you have figured out the answer confidently.

Question Title: <question_title>
Question Background: <background>
Expected Answer Type: <answer_type>

Think step by step about the information provided and put the answer to the question in <answer> </answer> tags and the earliest date on which the answer to the question could be inferred with certainty in <date> </date> tags. The date should be in the format YYYY-MM-DD.
Once you find the answer, please make sure to find THE EARLIEST DATE the answer could have been guessed. Try to search as much as possible across sites/pages to find out when was the earliest time the answer to the question was basically known/determined (or could have been inferred confidently from public knowledge).
\end{PromptVerbatim}
\end{promptframe}

\subsection{Answer Matching Evaluation Prompts}
\label{app:answer-matching-prompts}

This prompt is used at resolution time to decide whether a submitted free-form outcome matches the ground truth answer for scoring.

\begin{promptframe}{Resolved-answer equivalence prompt}
\begin{PromptVerbatim}
You are an objective judge of forecasting predictions.

Question: "<question_title>"
Predicted outcome: "<predicted_outcome>"
Ground truth (actual answer): "<ground_truth>"

Does the predicted outcome match the ground truth? Rules:
- YES if predictions are semantically equivalent (same meaning, different wording)
- YES if predicted outcome is MORE SPECIFIC than ground truth (e.g. "David Raya" matches "Raya")
- NO if predicted outcome contains generic text like "Unknown" or "Answer 1" or "Option 1"
- NO if predicted outcome is VAGUER/MORE GENERAL than ground truth (e.g., "a goalkeeper" does NOT match "David Raya")
- NO if they refer to different things

Essentially, you have to grade whether the forecaster correctly predicted the ground truth answer for the question.
Answer strictly "Yes" or "No".
\end{PromptVerbatim}
\end{promptframe}

This prompt is used before scoring to cluster semantically equivalent free-form predictions, so differently worded outcomes are treated as the same candidate answer.

\begin{promptframe}{Prediction-clustering match prompt}
\begin{PromptVerbatim}
You are an objective judge of forecasting predictions.

Question: "<question_title>"

New prediction: "<candidate_prediction>"

Existing predictions:
1. <existing_prediction_1>
2. <existing_prediction_2>
...
N. <existing_prediction_N>

Does the new prediction match any of the existing predictions semantically?
- Match if they mean the same thing or if new prediction is more specific
- Do NOT match if new prediction is vaguer/more general

If yes, respond with ONLY the number (e.g., "1" or "3").
If no match exists, respond with "None".

Answer:
\end{PromptVerbatim}
\end{promptframe}

\end{document}